\documentclass{article}
\pdfoutput=1
\PassOptionsToPackage{numbers}{natbib}
\usepackage[preprint]{neurips_2019}

\usepackage[utf8]{inputenc} % allow utf-8 input
\usepackage[T1]{fontenc}    % use 8-bit T1 fonts
\usepackage{hyperref}       % hyperlinks
\usepackage{url}            % simple URL typesetting
\usepackage{booktabs}       % professional-quality tables
\usepackage{amsfonts}       % blackboard math symbols
\usepackage{nicefrac}       % compact symbols for 1/2, etc.
\usepackage{microtype}      % microtypography
\usepackage{amsmath}
\usepackage{amsthm}
\usepackage[pdftex]{graphicx}
\graphicspath{ {./images_88_jpg/} }
\usepackage{subcaption}
\usepackage{wrapfig}
\hypersetup{hidelinks}

\newcommand{\R}{\mathbb{R}}
\newcommand{\X}{\mathcal{X}}
\newcommand{\Z}{\mathcal{Z}}
\newcommand{\N}{\mathcal{N}}
\newcommand{\Loss}{\mathcal{L}}
\newcommand{\Lossfeat}{\mathcal{L}_{mse\mbox{-}feat}}
\newtheorem{theorem}{Theorem}

\DeclareMathOperator*{\argmin}{arg\,min}

\title{Exploiting GAN Internal Capacity for High-Quality Reconstruction of Natural Images}

\author{%
  Marcos Pividori\thanks{This preprint is the result of the work done for
  the undergraduate dissertation of M. Pividori supervised by L.C. Uzal and G.L. Grinblat, 
  and presented in July 2019.}\\
  Universidad Nacional de Rosario\\
  \texttt{marcospividori@gmail.com}\\
  \And
  Guillermo L. Grinblat \quad Lucas C. Uzal\\
  CIFASIS, UNR-CONICET\\
  \texttt{\{grinblat,uzal\}@cifasis-conicet.gov.ar}
}

\begin{document}

\maketitle

\begin{abstract}
Generative Adversarial Networks (GAN) have demonstrated impressive results in modeling the distribution of natural images, learning latent representations that capture semantic variations in an unsupervised basis.
Beyond the generation of novel samples, it is of special interest to exploit the ability of the GAN generator to model the natural image manifold and hence generate credible changes when manipulating images. However, this line of work is conditioned by the quality of the reconstruction.
Until now, only inversion to the latent space has been considered, we propose to exploit the representation in intermediate layers of the generator, and we show that this leads to increased capacity. In particular, we observe that the representation after the first dense layer, present in all state-of-the-art GAN models, is expressive enough to represent natural images with high visual fidelity.
It is possible to interpolate around these images obtaining a sequence of new plausible synthetic images that cannot be generated from the latent space.
Finally, as an example of potential applications that arise from this inversion mechanism, we show preliminary results in exploiting the learned representation in the attention map of the generator to obtain an unsupervised segmentation of natural images.
\end{abstract}

\section{Introduction}
The development of Generative Adversarial Networks (GAN) \cite{goodfellow2014generative} represented a milestone in the problem of data generation, that is learning generative models that map samples from a simple latent distribution into samples of a complex data distributions such as natural images. In recent years several works have shown empirical results of applying these methods to different tasks. Recently, \citet{brock2018large} demonstrated that it is possible to train GAN models on a larger scale and generate high-resolution diverse samples from complex datasets such as ImageNet \cite{russakovsky2015imagenet}, with remarkable results.

The representation learned by the generator captures meaningful semantic variations along the data distribution \cite{radford2015unsupervised}.
Therefore, it is desirable to exploit this learned representation. In particular,
there has been a recent interest in making use of the capability of
the generator to approximate the manifold of natural images (semantic image editing) \cite{zhu2016generative} \cite{brock2016neural}, projecting real images into the latent space and manipulating them to produce smooth visual changes over high level features, preserving the realism of the result. This requires the initial step of obtaining a reconstruction close to the original image. However, the GAN framework lacks an automatic inference mechanism, which represents a bottleneck in these cases.
Inverting the generator is not a trivial operation, specially when it is a complex model.
Furthermore, it is frequently observed that real images cannot be represented in the latent space, obtaining approximations with a high reconstruction error, which limits the positive impact of being able to make semantic edits.

When there is no latent value that permits to reconstruct the image, it is generally considered to be evidence that the generator cannot model certain image attributes \cite{metz2016unrolled}, commonly referred to as mode dropping. In this work, we demonstrate that in many of these cases, modes can be modeled in the internal layers of the generator, but this capacity is not exploited from the latent space. In particular, we show the importance of the learned representation after the first dense layer of the generator, since it is very expressive and allows to represent arbitrary natural images with high visual fidelity. We also propose an inversion algorithm which permits to find meaningful representations,
in the sense that we can perform interpolation experiments around those images obtaining a full sequence of new valid synthetic images that cannot be generated from the latent space. By allowing to reconstruct much better the real images, our work has direct impact on all previous work on high-level image editing and processing using the GAN generator \cite{zhu2016generative} \cite{brock2016neural} \cite{yeh2017semantic}.

In addition, we demonstrate that a generator of the complexity of BigGAN \cite{brock2018large} can be inverted with a non-parametric approach while previous works
\cite{zhu2016generative}  \cite{creswell2018inverting} \cite{lipton2017precise} mostly consider simple DCGAN models and datasets of low variability.
Finally, as a new practical application, we show that it is possible to exploit the learned representation in the attention map of the generator to obtain an unsupervised segmentation of real images. %Finally, as an example of potential applications that arise from this inversion mechanism, we show preliminary results in exploiting the learned representation in the attention map of the generator to obtain an unsupervised segmentation of natural images.

\section{Related work}

Since the development of GAN there has been a general interest in inverting the generator and exploiting the representation that was learned in an unsupervised manner. For example, for retrieval and classification %of images
\cite{radford2015unsupervised}, to manipulate images  \cite{zhu2016generative} \cite{brock2016neural}, and to provides relevant insights on which features the generator has learned to model \cite{creswell2018inverting}.

Inverting the generator implies finding a vector $z \in \Z$ that when provided as input results in a image $G(z)$ that is very close to the target image. Mapping an image from pixel space to latent space is not a trivial operation, as it requires inverting the generator, which usually consists of a complex model of several non-linear layers. In addition, the same image could be generated from different $z$ values or none at all. In their original formulation, GANs do not provide a direct inference mechanism. In this direction, previous works can be grouped into two main approaches:

\paragraph{Parametric models}
A line of work proposes to learn a parametric model (encoder) that maps each image to a representation in the latent space $\Z$. \citet{donahue2016adversarial} and \citet{dumoulin2016adversarially} proposed to train the encoder jointly with the generator and discriminator, in an adversarial setup. Other works considered training the encoder on a pre-trained generator, through a regression in the $\Z$ space \cite{donahue2016adversarial}, or in the space of images \cite{luo2017learning}.

In general, these approaches have the disadvantage of requiring the training of a third model (encoder), increasing the number of parameters to be learned and the risk of overfitting or underfitting (depending on the training method).
Although good results have been shown when using the encoder as a feature extractor for classification tasks, in general reconstructions are not good, failing to preserve the structure and style of images. Finally, introducing a complex model to invert the generator makes it questionable as a diagnostic tool to evaluate the representation of the generator \cite{creswell2018inverting}.

\paragraph{Optimization on the generator}
\citet{creswell2018inverting} proposed to find the $z$ vector that generates a certain image $x$, solving an optimization problem over the generator. Essentially, they follow the gradient of the generator $G$ with respect to its input. 
\citet{zhu2016generative} includes a similar approach as part of larger framework to manipulate images.
\citet{lipton2017precise} proposed an extension to improve the recovery in cases of uniform prior distribution on the latent space. A similar algorithm was previously proposed by \citet{mahendran2015understanding} to study the representation of deep networks.

Our work focuses on the study of the intermediate representations of the generator, the invertibility of the different layers and the degree of reconstruction of real images. We opt to use the optimization approach on the generator, since it does not require the introduction and training of a new set of parameters for modelling the inverse function, which could weaken the conclusions that can be drawn from the results.

Unlike previous work, we conduct our experiments on ImageNet \cite{russakovsky2015imagenet}, which provides wider variability, including different classes of objects in different situations, and therefore the model has to deal with greater complexity. The methods mentioned above do not perform well or directly do not show results on this more complex dataset. Moreover, previous work on non-parametric approaches only consider simple DCGAN models, while we demonstrate that we can invert much deeper generators (BigGAN), which makes the inversion much more challenging. In addition, we propose to extend this inversion mechanism to the hidden layers of the generator, showing the advantages of using the learned representation in the first fully connected layer, where it is possible to reconstruct natural images with high visual fidelity while still capturing high-level features.

\section{Inverting the generator to intermediate layers}

Let $G: \Z \rightarrow \X$ be a generator of deep architecture, composed of $n$ layers, that transforms a latent vector $z\in\R^{d_z}$ sampled from a distribution $z \sim P_z$ to an image $\hat{x} = G(z)$, with an implicit distribution $P_{model}$ that approximates $P_{data}$, the distribution of real data on the image space $\X = \R ^{W \times H \times C}$.

We can split the generator in a given hidden layer $l$ (%of certain
with dimensionality $d_l$), and analyze the learned representation at that point, redefining the generator as the composition of two generators:
\[
    G(z) = G^l_2 ( G^l_1 ( z ) )
\]
where $G^l_1 : \Z \rightarrow \R^{d_l}$ represents the transformation from latent space to layer $l$, and $G^l_2 : \R^{d_l} \rightarrow \X$ from layer $l$ to the space of images $\X$.

Let $P^l_h$ be the generated distribution in the hidden layer $l$ ($\R^{d_l}$) according to the random variable $H^l_{gen}=G^l_1(z)$ with $z \sim P_z$. Then, we can consider $G^l_2(h)$ as a generator from the learned latent distribution $h \sim P^l_h$, which has the same performance as the original generator $G(z)$ with $z \sim P_z$.

%In general, the distribution $P_z$ on the latent space is a simple Gaussian or Uniform, where each dimension is independently sampled. As a result, the learned features in the latent representation are not correlated. On the other hand, the generated distribution $P_h^l$ at a given layer $l$ can be much more complex where the support of $P_h^l$ is contained in a manifold $G_1^l(\Z)$ on the $\R^{d_l}$ space, and the learned features associated to different dimensions can be highly correlated.

\subsection{Invertibility of the generator}

We say that $G$ is invertible in a set of images $S$ if it is right invertible, that is, if there is a function $G^{-1} : S \rightarrow \Z$ such that $G(G^{-1}(x)) = x\ \ \forall x\in S$. An optimal generator, that is, that can generate the target distribution in the output space ($P_{model} = P_{data}$), is invertible on real images.
 
\begin{theorem}
An optimal generator $G^*$ is right invertible $P_{data}$-almost everywhere in $\X$.
\end{theorem}
\begin{proof}
Let $\X' = \X - G^*(\Z)$ be the set of images that can not be generated. Then $\X'$ has measure zero under $P_{data}$, therefore $G^*$ is right invertible $P_{data}$-almost everywhere in $\X$.
\[
P_{data}(\X') = P_{model}(\X') = 0 \qedhere
\]
\end{proof}
For a non-optimal generator, we expect real images to be approximately invertible, where the reconstruction quality depends on the layer considered for the representation. At the output layer, it is possible to represent all images with zero reconstruction error. As we move in the network backwards (considering the input space of $G_2^l$ for some layer $l$), the reconstruction error for certain images is expected to increase.

Since in the generator the information flows forward, every image represented in a given layer has no better representation in previous layers, and at least as good in the following layers. If we consider two layers $l < m$, the image set of $G_2^{l}$ is a subset of the image set of $G_2^{m}$ ($G_2^{l}(\R^{d_l}) \subseteq G_2^{m}(\R^{d_{m}})$). Then, the closer to the latent space, the more restricted the image set $G_2^l(\R^{d_l})$ at pixel level, which means that more images are filtered out. It would be expected that in a properly trained generator, the increase in reconstruction error is manifested mainly in irrelevant areas of the distribution at pixel level (e.g. white noise images) and to a much less extent in natural images.

%The manifold [27] [28] hypothesis postulates that real data in a high dimensional representation reside in a low dimensional manifold embedded in a high dimensional space. In [29] it is hypothesized and empirically shown that deeper representations can better disentangle the factors that explain data variability. And more disentangled representations make it possible to unfold the manifolds on which the data are concentrated while expanding the relative volume occupied by the high probability points near these manifolds.

%In the case of a generator trained to generate real images, deeper representations would be expected to capture variability factors in the images, and real images would be expected to occupy a larger relative volume, and conversely, low-probability images of the real distribution, such as white noise, would occupy less space.

% For each layer que can distinguish three concepts:
% + The Space R^d.
% + The Image Space G_1(\Z).
% + The Generated Distribution in the Image Space G_1(z) z~P_z.

\subsection{First dense layer}

In particular, in this work we are interested in analyzing the learned representation in the first dense layer (included in all the state-of-the-art GAN architectures). Suppose the first layer of the generator has $d_{1}$ fully connected units, then we split the generator at this point:
\[
    G_1^{1}(z) = W_1 z + b_1
\]
 where $W_1 \in \R^{d_{1} \times d_z}$. $G_1^{1}(\Z)$ represents a linear subspace of dimensionality at most $d_z$ (column space of $W_1$) in $\R^{d_{1}}$. In general $d_{1}$ is one or two orders of magnitude greater than $d_z$.

From now on, we omit the index $l$ as we will always consider the first fully connected layer (e.g. $G_1$ denotes the first linear mapping $G_1^{1}$ and $G_2$ the rest of the network $G_2^{1}$), although most of the concepts can be extended to other intermediate layers. The output space $\R^{d_1}$ is referred to as \textit{the space of the dense layer}.

\subsection{Proposed algorithm}

\begin{figure}
  \centering
  \includegraphics[width=0.95\linewidth]{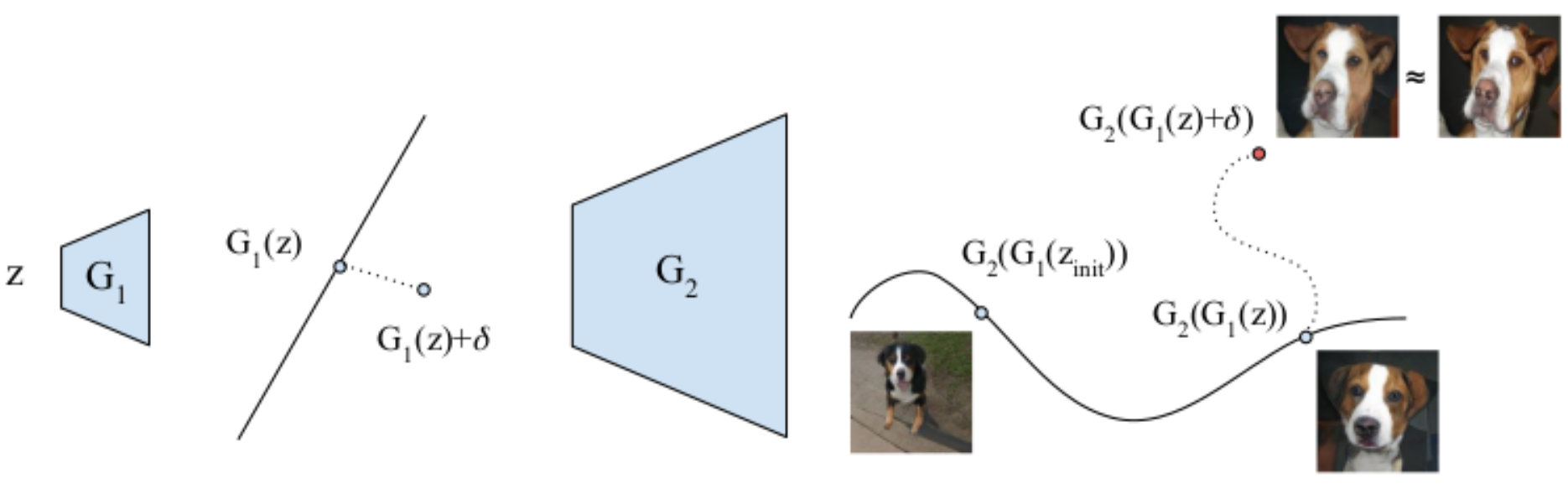}
  \caption{Interpretation of the two-step optimization at different levels of representation. In the space of the first layer, the linear subspace $G_1(\Z)$ is considered first, then the displacement $\delta$ permits to explore the whole space. In pixel-level, first we find the best approximation in the generated manifold $G(\Z)$ and then $\delta$ involves considering the directions captured by the manifold $G_2(\R^d)$.}
  \label{fig:optimization}
\end{figure}

Given a generator $G$ whose computational graph is differentiable, we want to infer a representation $h$ in the space of the first layer that makes it possible to reconstruct a target image $x \in \X$, guided by the gradient $\nabla_h \Loss(G_2(h),x)$.
\space $\Loss$, the loss to minimize, represents the reconstruction error between the generated and target image. In other words, we approximate the following optimization by gradient descent:
\[
    h^* = \argmin_{h \in \R^d} \Loss(x, G_2(h)).
\]
The space of the first dense layer is of a much greater dimensionality than that of the generated linear subspace $G_1(\Z)$, therefore the capacity of representation is considerably more extensive. When optimizing in the space of the dense layer without regularization, the algorithm obtains representations that, although map to the target image, are not related to the distribution of points generated in this space of the network during training. As a result, the obtained representation is poor and does not seem to capture the factors of variation in the data. For example when shifting the vector in any direction, the image degrades.

Instead, it is desirable to choose, among all possible representations of an image, the one closest to the generated distribution ($P_h$), where the generator was adjusted to output plausible samples.
Then, we propose to optimize on $G_2(G_1(z)+\delta)$ where $\delta$ represents a displacement in the first layer with respect to the point $G_1(z)$. The optimization is split into two steps, first the reconstruction error on the latent vector is optimized:
\[
    z^* = \argmin_{z \in \Z} \Loss(x,G_2(G_1(z))) - \lambda_1\ \log\ P_z(z)
\]
Then, it is optimized over a displacement in the full space of the dense layer:
\[
\delta^* = \argmin_{\delta \in \R^d}\Loss(x, G_2(G_1(z^*)+\delta)) + \lambda_2\ \left\lVert \delta\right\rVert_1
\]
Obtaining the final representation $h^* = G_1(z^*) + \delta^*$.

\paragraph{Regularization in z.}
The term $\log\ P_z(z)$ represents the log likelihood of $z$, and regularizes the search in the latent space to ensure that the optimization stays within probable regions of $\Z$, as proposed in \cite{creswell2018inverting}. For example, if the network is trained with $z \sim \N(0, I)$, the log likelihood reduces to a regularization in $\left\lVert z\right\rVert_2^2$.

\paragraph{Regularization in $\delta$.}
In a space of such high dimensionality as the first layer, we choose the $l_1$ norm to regularize $\delta$ because it encourages sparse solutions, including the least amount of non-zero components necessary to properly approximate the target image.

We can analyze this two-step optimization in the space of the dense layer. First, when optimizing up to the latent space, only points of the linear subspace $G_1(\Z)$ are considered.
Moreover, if the first linear mapping $G_1$ is injective (a reasonable assumption), optimizing $G_2(G_1(z))$ with a regularization term in the log likelihood of $z$ is equivalent to optimizing $G_2(h)$ over the entire space $\R^d$ of the dense layer, with a regularization in the log likelihood of $h$ with respect to the generated distribution $P_h$.

\begin{theorem}
If the first linear mapping $G_1$\ is\ injective:
\[
\argmin_{h \in \R^d}\Loss(x,G_2(h)) - \lambda_1 \log\ P_h(h) = G_1(\argmin_{z \in \Z} \Loss(x,G_2(G_1(z))) - \lambda_1 \log\ P_z(z))
\]
\end{theorem}
\begin{proof}%\raisegroup
\begin{align*}
&\quad\ \argmin_{h\in \R^d} \Loss(x, G_2(h)) - \lambda_1 \log\ P_h(h) \\ %\span
&= \argmin_{h\in G_1(\Z)}\Loss(x, G_2(h)) - \lambda_1 \log\ P_h(h)  \tag{$P_h(h)= 0\ \ \forall h \in \R^d - G_1(\Z)$}\\
&= \argmin_{h\in G_1(\Z)}\Loss(x, G_2(G_1(G_1^{-1}(h)))) - \lambda_1 \log\ (P_z(G_1^{-1}(h)) \ \frac{1}{|det(\frac{\partial G_1(G_1^{-1}(h))}{\partial G_1^{-1}(h)})|}) \tag{$G_1$ injective}\\
&= \argmin_{h\in G_1(\Z)}\Loss(x, G_2(G_1(G_1^{-1}(h)))) - \lambda_1 \log\ \frac{P_z(G_1^{-1}(h))}{|det(W_1)|}\\ %\tag{$G_1$ injective}\\
&= \argmin_{h\in G_1(\Z)}\Loss(x, G_2(G_1(G_1^{-1}(h)))) - \lambda_1 \log\ P_z(G_1^{-1}(h)) + \lambda_1\ \log\ |det(W_1)|\\
&= \argmin_{h\in G_1(\Z)}\Loss(x, G_2(G_1(G_1^{-1}(h)))) - \lambda_1 \log\ P_z(G_1^{-1}(h))\\
&= G_1(\argmin_{z\in \Z}\Loss(x, G_2(G_1(z))) - \lambda_1 \log\ P_z(z)) && \qedhere
\end{align*}
\end{proof}

Then, in the second step of the optimization, the displacement $\delta$ exploits the internal capacity of the generator, considering all the space in the first layer ($\R^d$), permitting to combine the features independently, out of the linear restriction established by the mapping $G_1$. Note that this involves generating images that cannot be generated from the latent space.

The coefficient $\lambda_2$ represents a compromise between the quality of reconstruction and the quality of the obtained representation. Larger values of $\lambda_2$ mean that the obtained representation remains close to the generated distribution, but the reconstruction may not be that close to the target image.
On the other hand, as the value of $\lambda_2$ decreases, the optimization displaces to a greater extent over the whole space of the first layer and therefore, even though the reconstruction ($\hat{x}$) resembles the original image ($x$), the obtained representation ($h$) may not be as meaningful.

The two-step optimization can be interpreted in the output space of the generator (Figure \ref{fig:optimization}). First we look for the best approximation of the target image considering $d_z$ directions in the generated manifold $G(\Z)$. Then, starting from this point, we extend the search to all the directions along the manifold $G_2(\R^d)$, looking for the simplest way to get to the target image, that is, the one that involves moving in the least amount of possible directions.

%With this two-step optimization, much more significant representations are obtained, as can be observed in the next section.% when interpolating between the representation of different inverted images.

%We can do the same analysis of the optimization, but at pixel-level. The image set of the generator $G(\Z)$, represents a manifold of at most $d_z$ directions in pixel space that approximates the manifold of real images. However, this manifold does not cover all modes, since real images in general cannot be perfectly reconstructed from the latent space.

%On the other hand, $G_2(\R^d)$ also represents a manifold in the image space, but with many more directions, which can cover more modes. However, by being so expressive, it also covers many irrelevant images.

\section{Experiments with the BigGAN generator}

Prior work on inverting GAN generators with non-parametric approaches only consider simple DCGAN models and datasets of low variability \cite{creswell2018inverting} \cite{lipton2017precise}. When attempting to invert more complex models such as BigGAN, the problem is more challenging, as the function to be optimized ($\Loss(x, G(z))$) seems to be highly non-convex, conditioning the result of the optimization. In the following experiments \footnote{Source code is available at: \url{https://github.com/CIFASIS/exploiting-gan-internal-capacity}}, we consider the pre-trained class-conditional BigGAN model trained on ImageNet for an image resolution of 128x128, publicly available\footnote{\url{https://tfhub.dev/deepmind/biggan-128/2}}.

\subsection{Reconstruction error}

We found in preliminary experiments that by simply optimizing over the Mean Square Error (MSE) at pixel level it is not possible to recover the representation of generated images ($z$), since the optimization frequently gets stuck in non-optimal critical points.
Metrics that compare images pixel-by-pixel, such as MSE, do not capture very well the similarity of images according to our perception of natural images \cite{wang2009mean}. As proposed in \cite{johnson2016perceptual} \cite{dosovitskiy2016generating}, some invariance to irrelevant transformations and at the same time sensitivity to important image properties, such as edges and textures, can be achieved through the use of metrics in a higher level feature space, %such as those
extracted by convolutional networks.

Let $C_l$ represent the activation at the layer $l$ of the InceptionV3 \cite{szegedy2016rethinking} network trained to classify ImageNet, then we propose to compare two images considering the euclidean distance in the feature space: $\left\lVert  C_l(x) - C_l(\hat{x}) \right\rVert_2$.
Given that different layers capture different levels of complexity \cite{goodfellow2009measuring}, after considering the layers of the InceptionV3 network, best results where achieved when comparing images at the representation of layer 7 (\textit{Mixed\_7a}), with a compromise between generalization and preservation of the visual structure of the images. Then, we define the reconstruction error as a linear combination of the MSE and the distance in the feature space extracted by the InceptionV3 network:
\begin{center}
  $\Lossfeat(x,\hat{x}) = \left\lVert x-\hat{x} \right\rVert _2^2 + \lambda_{feat}\ \left\lVert \ C(x) - C(\hat{x}) \right\rVert_2^2$
\end{center}
Following the hypothesis \cite{bengio2013better} that deeper representations unfold the manifolds on which real data is concentrated, one could hypothesize that the representation learned by a deep network trained on the set of real images (for example InceptionV3) succeeds in disentangling the manifold of natural images and, therefore, also the generated manifold $G(\Z)$ that attempts to approximate it. This is consistent with the empirical results: the euclidean distance at the level of features extracted by a deep network provides better gradients that allow to traverse the generated manifold until recovering the latent vector of generated images.

\subsection{Inverting generated images}

\begin{figure}
  \centering
  \begin{subfigure}{0.32\textwidth}
  \includegraphics[width=\linewidth]{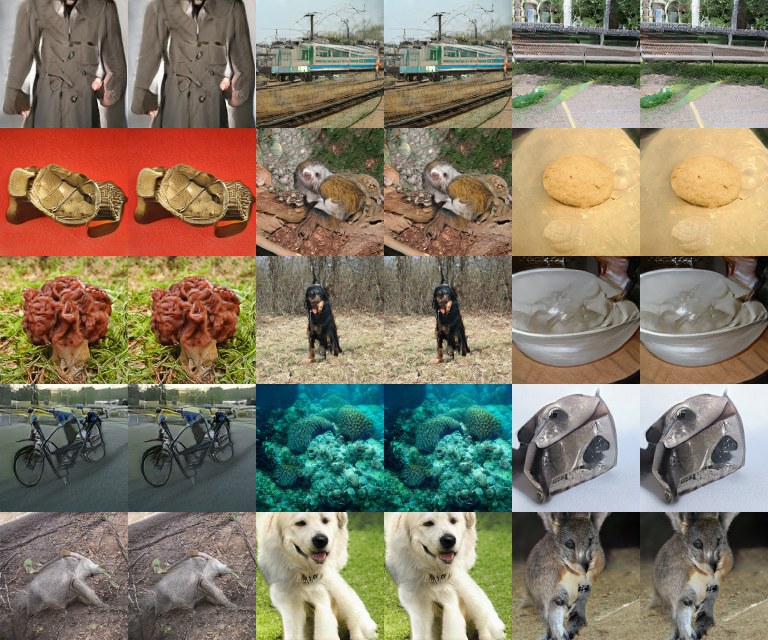}
  \caption{Generated images (latent)} \label{fig:latent_fake}
  \end{subfigure}
  \hspace*{\fill}
  \begin{subfigure}{0.32\textwidth}
  \includegraphics[width=\linewidth]{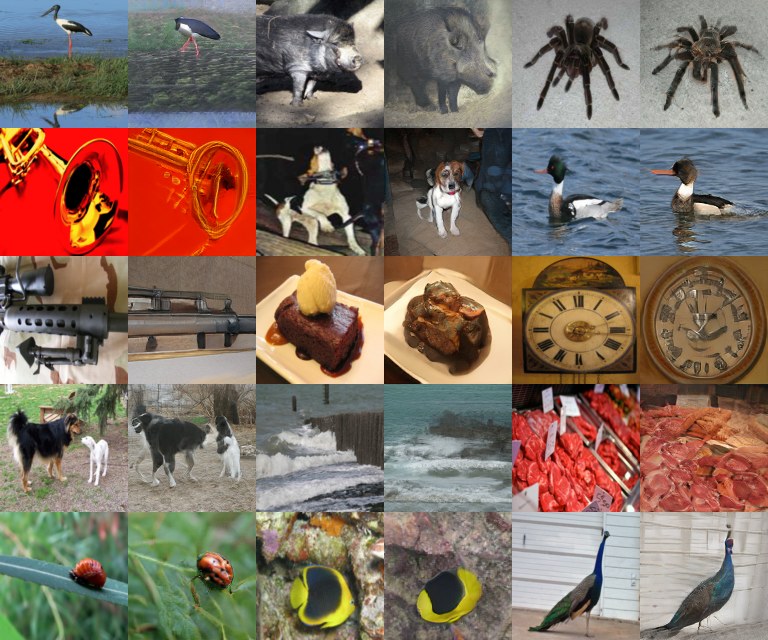}
  \caption{Real images (latent).} \label{fig:latent_real}
  \end{subfigure}
  \hspace*{\fill}
  \begin{subfigure}{0.32\textwidth}
  \includegraphics[width=\linewidth]{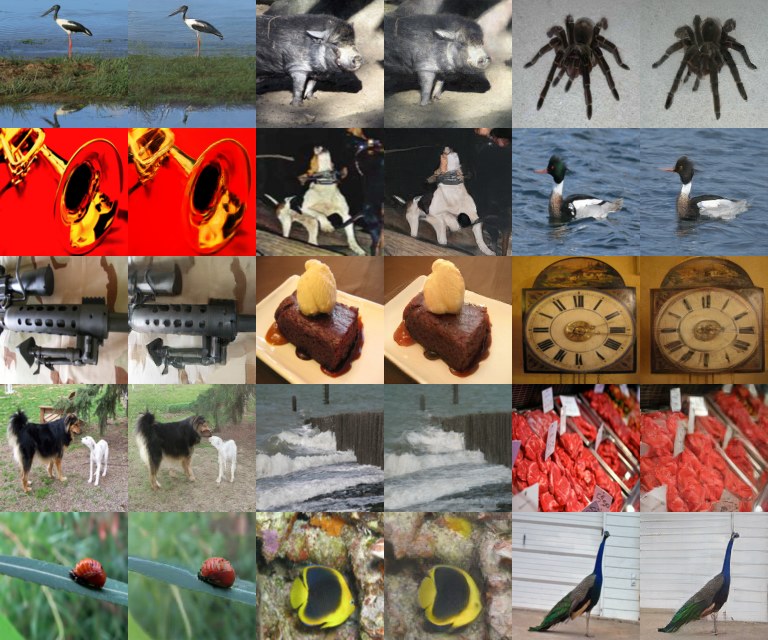}
  \caption{Real images (dense).} \label{fig:dense}
  \end{subfigure}
  \caption{Results of inverting of BigGAN for a random sample of images. Generated images can be reconstructed with high fidelity to the latent space (a). For real images, reconstructions to the latent space (b) are semantically related to the target image. The quality of the reconstructions is improved when considering the representation in the first dense layer (c).}
\end{figure}

\begin{figure}
  \centering
  \includegraphics[width=0.49\linewidth]{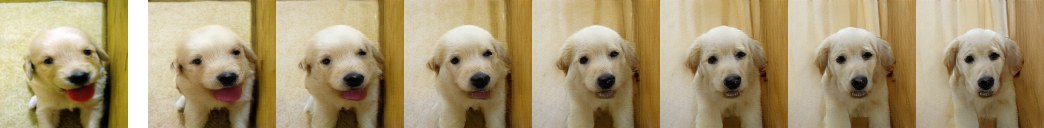}
  \hspace*{\fill}
  \includegraphics[width=0.49\linewidth]{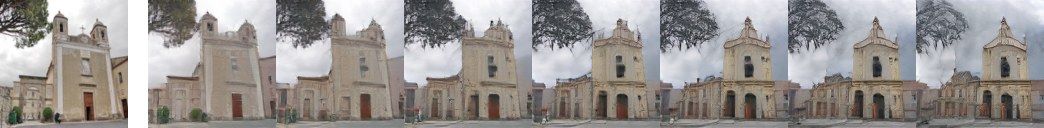}
  \includegraphics[width=0.49\linewidth]{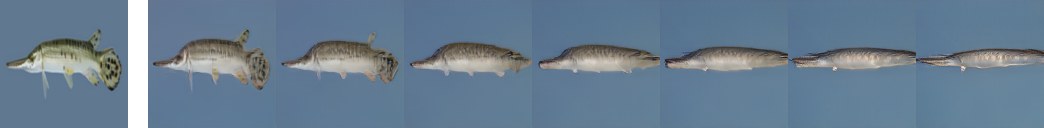}
  \hspace*{\fill}
  \includegraphics[width=0.49\linewidth]{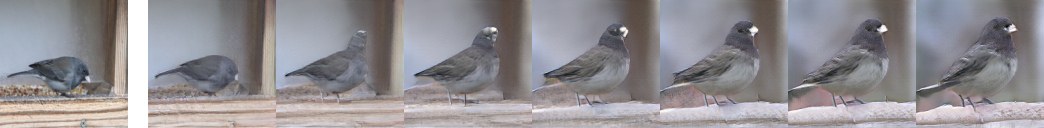}
  \includegraphics[width=0.49\linewidth]{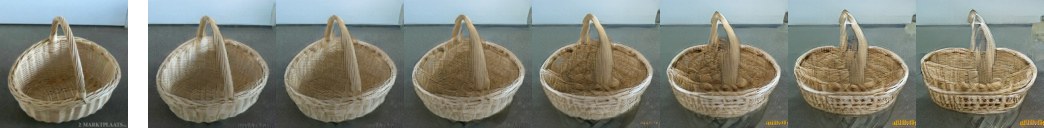}
  \hspace*{\fill}
  \includegraphics[width=0.49\linewidth]{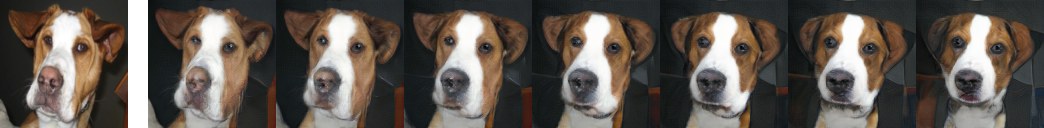}
  \includegraphics[width=0.49\linewidth]{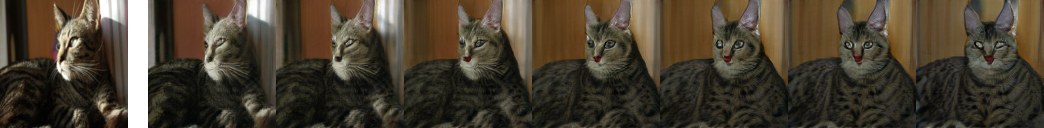}
  \hspace*{\fill}
  %\includegraphics[width=0.49\linewidth]{regenerate_287}
%  \includegraphics[width=0.49\linewidth]{regenerate_338}
%  \hspace*{\fill}
  \includegraphics[width=0.49\linewidth]{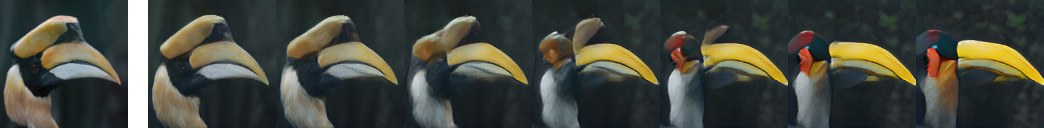}
  \caption{Left: Real image. Right: Linear interpolation between its closest reconstruction in the first step of the optimization ($G_2(G_1(z^*))$, right) and the closest reconstruction in the second step ($G_2(G_1(z^*)+\delta^*)$, left). Note that except from the right-most column, the rest of intermediate images can not be generated from the latent space.}
  \label{fig:interpolation}
\end{figure}

\begin{figure}
  \centering
  \includegraphics[width=0.49\linewidth]{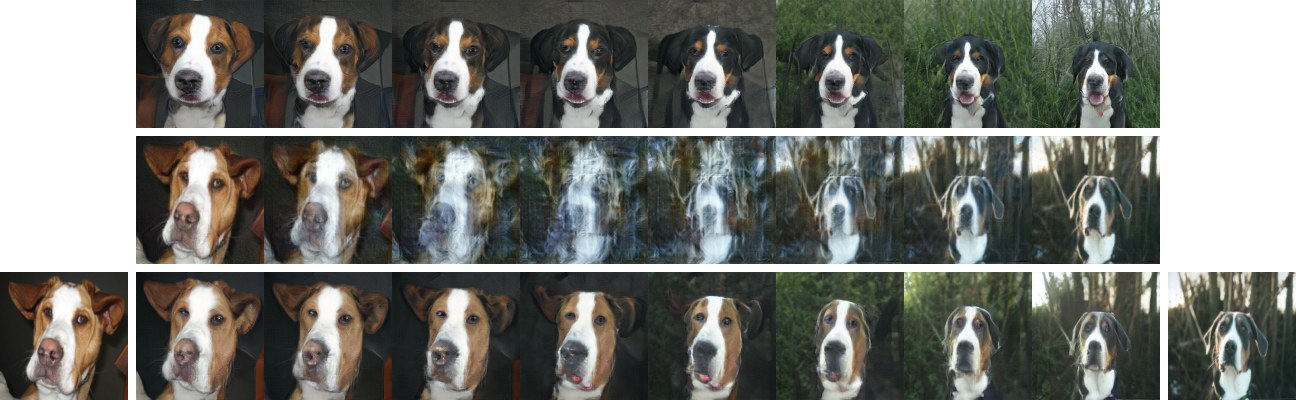}
  \hspace*{\fill}
  \includegraphics[width=0.49\linewidth]{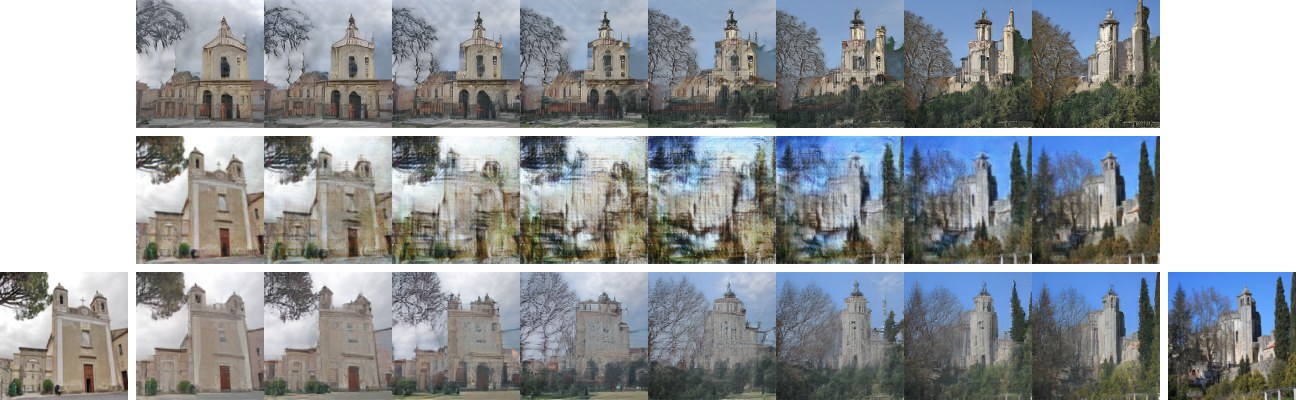}
  \caption{Linear interpolation between the reconstruction of two real images for the same class. First row: reconstruction in the latent space. Second row: reconstruction in the dense layer without regularization. Third row: reconstruction in the dense layer with the proposed two-step optimization.}
  \label{fig:linear_interp}
\end{figure}

Before drawing conclusions about the representational power of different layers of the network, it is important to ensure that the inversion mechanism works correctly, verifying that it is possible to recover generated images for which exists an input value that perfectly maps to the target image.

Considering the generated distribution $P_{model}$ as a baseline, 1000 images were generated ($x \sim P_{model}$) and inverted to the latent space minimizing the reconstruction error $\Lossfeat$ with a regularization on the log likelihood. As can be observed in Figure \ref{fig:latent_fake}, it is possible to reconstruct the images to an insignificant reconstruction error which confirms that the inversion algorithm is working properly.

\subsection{Inverting natural images}

Next, the distribution of natural images $P_{data}$ is considered, sampling 1000 random images of ImageNet ($x \sim P_{data}$), and inverting them following the proposed two-step optimization on the reconstruction error $\Lossfeat$.

In the first step of the optimization (up to the latent space) we obtain reconstructions that are semantically related to the target image although they differ substantially (Figure \ref{fig:latent_real}). Since the inversion mechanism works correctly for generated images, we can conclude that a high reconstruction error is obtained for real images because they cannot be represented in the latent space.
% that the high reconstruction error for real images is due to the fact that they cannot be represented in the latent space.
%we can conclude 
When the second step of the optimization is performed, including a displacement $\delta$ over the entire space of the dense layer, the reconstruction error is considerably reduced, obtaining high-quality reconstructions (Figure \ref{fig:dense}).

\begin{wrapfigure}{r}{0.5\textwidth}
%\begin{figure}
  \centering
  \includegraphics[width=0.95\linewidth]{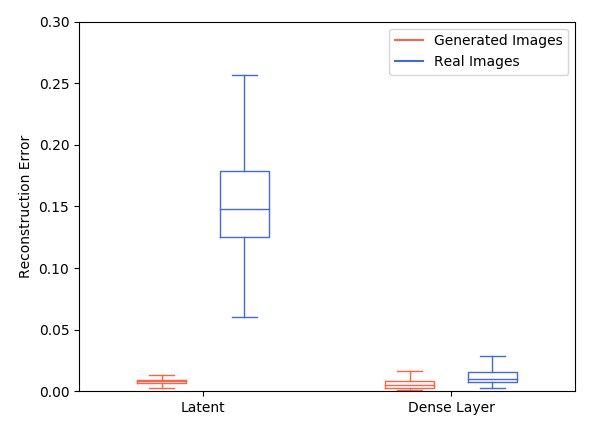}
  \caption{Reconstruction error ($\Lossfeat$) for a random sample of 1000 generated and 1000 real images, at different levels of representation.}
  \label{fig:boxplot}
%\end{figure}
%\vspace{-5mm}
\end{wrapfigure}

%As shown in Figure \ref{fig:boxplot}, there is a significant gap in the reconstruction error over real images between the two levels of the network, demonstrating the difference in the power of representation of natural images. In the first dense
As shown in Figure \ref{fig:boxplot}, there is a significant gap between the reconstruction error over real images at the two levels of the network, demonstrating the difference in the power of representation of natural images. In the first dense layer, natural images can be represented with a low reconstruction error, similar to generated images.

At the same time, representations obtained in the first layer are meaningful, as can be seen when interpolating linearly with generated images (Figure \ref{fig:interpolation}) and with the representation of other natural images of the same class (Figure \ref{fig:linear_interp}), producing a smooth transition. It should be noted that all intermediate images resulting from these interpolations cannot be generated from the latent space, as they are outside the linear subspace $G_1(\Z)$. This demonstrates that the generator has a greater exploitable capacity to generate images than what is captured by the latent space.

\section{Unsupervised segmentation}

Inverting the generator on natural images means learning to generate them, and the structure of the network can provide information on what the components of the images are and how they relate to each other.
As an example of potential application, we show that the learned representation in the self-attention map of the generator \cite{zhang2018self} \cite{wang2018non} can be exploited to obtain an unsupervised segmentation of real images. We consider the Embedded Gaussian variant of the non-local block \cite{wang2018non}, as employed in BigGAN. Given input feature vectors $x_i$ ($i$ index over $N$ spacial positions) the attention matrix $A \in \R^{N \times N}$ is defined as:
\[
    A_{ij} = \dfrac{e^{\psi(x_j)^T\ \phi(x_i)}}{\sum_je^{\psi(x_j)^T\ \phi(x_i)}}
\]
where $\psi(x)=W_\psi\ x$ and $\phi(x)=W_\phi\ x$ are two embedding spaces. $A_{ij}$ indicates the attention on position $j$ when computing the output of position $i$. Using this learned attention map, we propose to define a dissimilarity matrix $D \in [0,1]^{N \times N}$ between different points of the image:
\[
    D = \left( 1 - \dfrac{ A + A^T }{2}\right) \odot (1-I).
\]
In this way, the matrix is symmetric and the dissimilarity of a point with itself is minimum.

In particular, for BigGAN on 128x128 resolution, the non-local block is included at 64x64 resolution, and contains a max pooling intermediate layer (Subsampling Trick \cite{wang2018non}). A 32x32 attention map is computed for each point of the 64x64 resolution, determining how much attention to pay to each section of the 32x32 resolution. Therefore, we upsample it spatially to a 64x64 map, and then compute the dissimilarity matrix.

Then, we apply Agglomerative Hierarchical Clustering to the dissimilarity matrix $D$, with average linkage as link criteria between different clusters. This can be interpreted as joining clusters on the basis of the average attention paid to each other. Resulting clusters represent an unsupervised segmentation of the image at the resolution of the attention map (64x64 in our example).

Therefore, given a high quality reconstruction of a real image, it is possible to segment the image using the structure of the network. As shown in the Figure \ref{fig:segmentation}, although not perfect, resulting clusters group significant sections of the image associated with the same concept.

\begin{figure}
  \centering
  \includegraphics[width=0.3\linewidth]{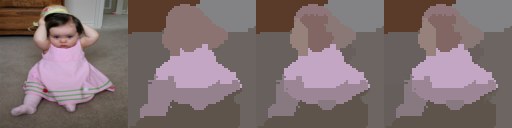}
  \includegraphics[width=0.3\linewidth]{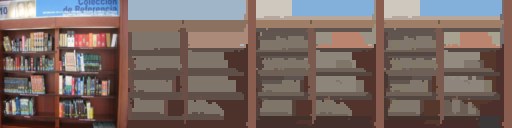}
  \includegraphics[width=0.3\linewidth]{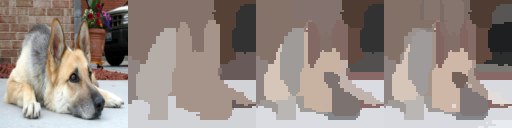}
  \includegraphics[width=0.3\linewidth]{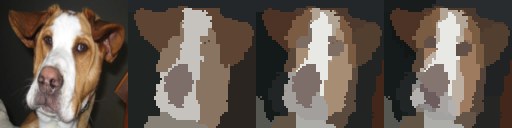}
  \includegraphics[width=0.3\linewidth]{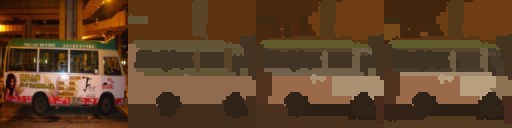}
  \includegraphics[width=0.3\linewidth]{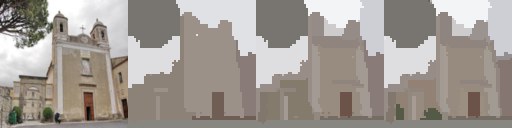}
  \includegraphics[width=0.3\linewidth]{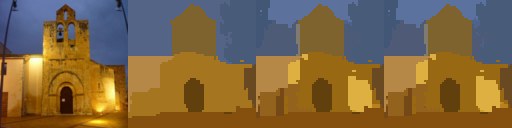}
  \includegraphics[width=0.3\linewidth]{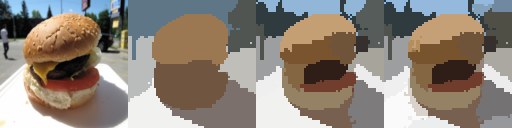}
  \includegraphics[width=0.3\linewidth]{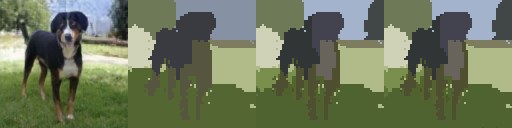}
  \caption{Examples of unsupervised segmentation (64x64) for real images, after inverting the generator to the first layer (different numbers of clusters: 8, 20, 40). Only for visualization purposes, each cluster is associated to the average color of all its pixel members.}.
  \label{fig:segmentation}
\end{figure}

\section{Conclusion}

The possibility of accessing to high-level representations of natural images and their respective reconstructions allows for countless image processing and editing applications based on manipulation of semantic features. The success of these applications is limited by the quality of the reconstruction and the quality of the reached representation.

In this work we have shown that it is possible to reach \textit{good} representations of real natural images in the space after the first layer of a GAN generator. These are good representations in the sense that, on one hand, we can obtain high-quality reconstructions of the images and, on the other hand, we can perform interpolation experiments in the representation space obtaining a full sequence of plausible images from any real image to a target image (real o generated).

All the experiments were performed using the state-of-the-art BigGAN generator and the ImageNet dataset, which is the most challenging scenario. Finally, as a further example of potential application, we also showed that the learned representation in the self-attention map of the generator can be exploited to obtain an unsupervised segmentation of real images.

%@TODO: Habría que decir como futura linea de trabajo estudiar más en profundidad el espacio de la primer capa, y las direcciones de los deltas. Si se puede samplear de alguna manera automático y explotar la capacidad interna del generador para samplear.

%NO VA EN LA VERSION ANONIMA
\subsubsection*{Acknowledgments}
The authors acknowledge grant support from ANPCyT PICT-2016-4374 and ASaCTeI IO-2017-00218. The Titan Xp used for this research was donated by the NVIDIA Corporation.

%\section*{References}
%\small
\bibliographystyle{unsrtnat}
\bibliography{bib}{}

\newpage

\begin{center}
\large{Supplementary Material of}\\
\vspace{2mm}
\huge{Exploiting GAN Internal Capacity for
High-Quality Reconstruction of Natural Images}
\end{center}

\vspace{2mm}
Source code for reproducing the experiments is available at:
\begin{center}
\url{https://github.com/CIFASIS/exploiting-gan-internal-capacity}
\end{center}

\section*{Linear interpolation in the space of the dense layer}
We show more examples of linear interpolation in the space of the dense layer,
as a demonstration of the quality of obtained representations
with the proposed two-step optimization. Note that all intermediate images of
these interpolations can not be generated from the latent space.

\paragraph{Representation gap.}
Figure \ref{fig:delta_interpolation} shows examples of linear interpolation
between the best reconstruction in the latent space and the best reconstruction
in the space of the dense layer: %$G_1(z^*)$ and $G_1(z^*+\delta^*)$.
\[
G_2(G_1(z^*) + \alpha\ \delta^*) \quad \text{with}\ \alpha \in [0,1].
\]

\paragraph{Intra class interpolation.}
Figure \ref{fig:linear_interp_sup} shows examples of linear interpolation
between two reconstructions of different real images ($x_1$ and $x_2$) in the
same class:
\[
G_2(\alpha\ h_1 + (1 - \alpha)\ h_2)  \quad \text{with}\ \alpha \in [0,1]
\]
where $h_1 = G_1(z_1^*) + \delta_1^*$ and $h_2 = G_1(z_2^*) + \delta_2^*$
are the obtained representations of $x_1$ and $x_2$
in the space of the dense layer.

\paragraph{Random interpolation.}
Figure \ref{fig:random_interpolation} shows examples of linear interpolation
between the inverted real image %($G_1(z^*+\delta^*)$)
and random generated images in the same class: % and $G_1(z_{rand})$
\[
G_2(\alpha\ h_1 + (1 - \alpha)\ h_2) \quad \text{with}\ \alpha \in [0,1]
\]
where $h_1 = G_1(z^*) + \delta^*$ and
$h_2 = G_1(z)$, $z \sim \N(0,1).$

\section*{Unsupervised segmentation}
Figure \ref{fig:segmentation_sup} includes more examples of segmented real images,
clustering the
self-attention map in the non-local block of BigGAN \cite{brock2018large}.
First, images are
inverted to the dense layer with the two-step optimization, obtaining a
representation $h^* = G_1(z^*)+\delta^*$.
Let $G_{att} : \R^d \rightarrow \R^{N \times N}$ represent the mapping
from the dense layer to the attention map. Then, we cluster the output of
$G_{att}(h^*)$ and generate a segmentation of the image.

\begin{figure}
  \centering
  \includegraphics[width=0.49\linewidth]{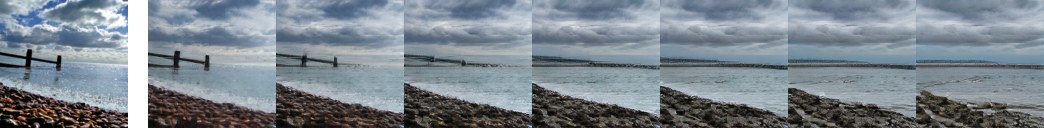}
  \hspace*{\fill}
  \includegraphics[width=0.49\linewidth]{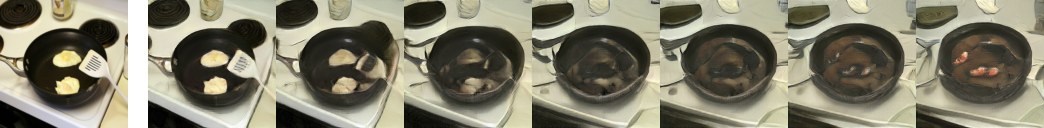}
  \includegraphics[width=0.49\linewidth]{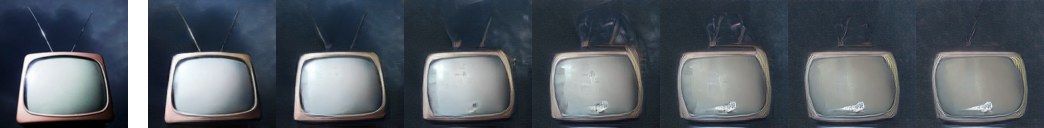}
  \hspace*{\fill}
  \includegraphics[width=0.49\linewidth]{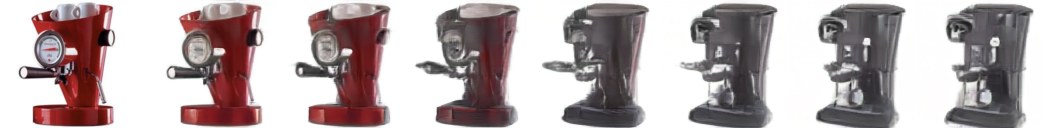}
  \includegraphics[width=0.49\linewidth]{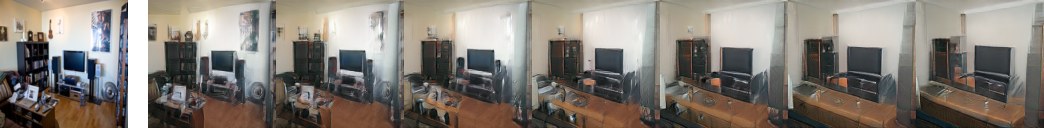}
  \hspace*{\fill}
  \includegraphics[width=0.49\linewidth]{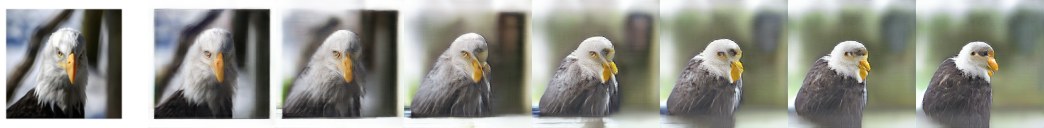}
  \includegraphics[width=0.49\linewidth]{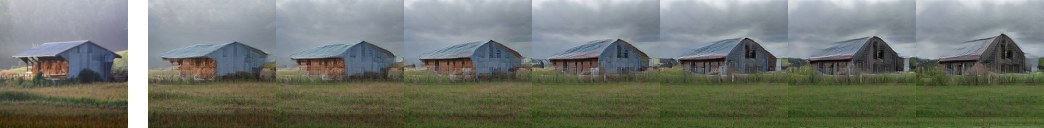}
  \hspace*{\fill}
  \includegraphics[width=0.49\linewidth]{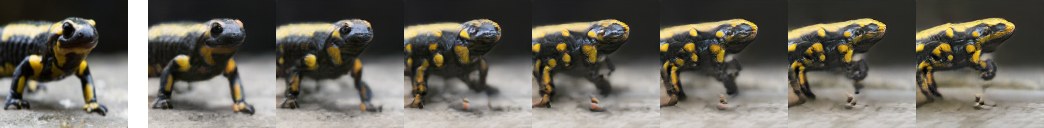}
  \includegraphics[width=0.49\linewidth]{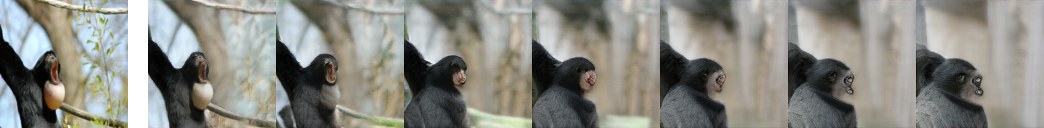}
  \hspace*{\fill}
  \includegraphics[width=0.49\linewidth]{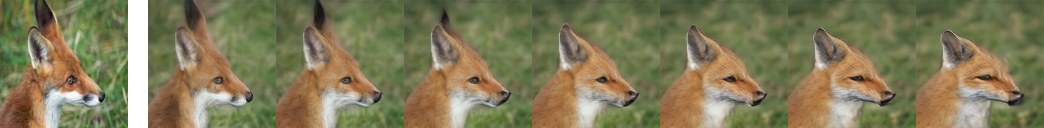}
  \includegraphics[width=0.49\linewidth]{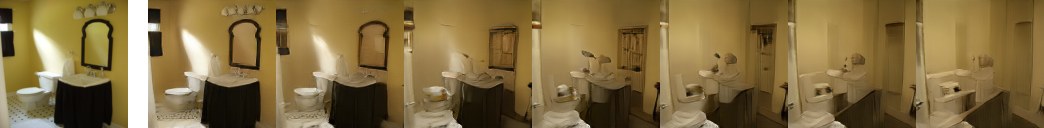}
  \hspace*{\fill}
  \includegraphics[width=0.49\linewidth]{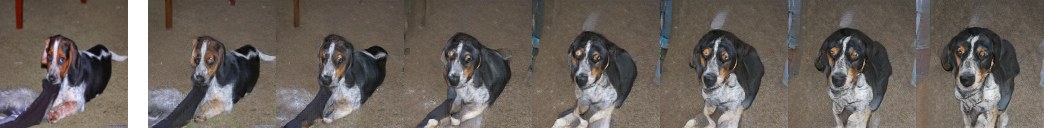}
  \includegraphics[width=0.49\linewidth]{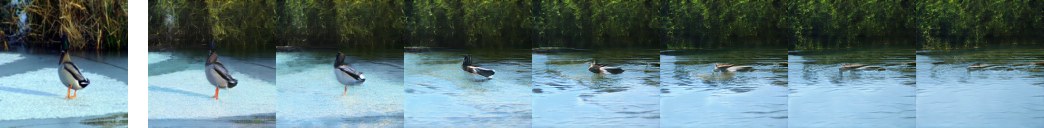}
  \hspace*{\fill}
  \includegraphics[width=0.49\linewidth]{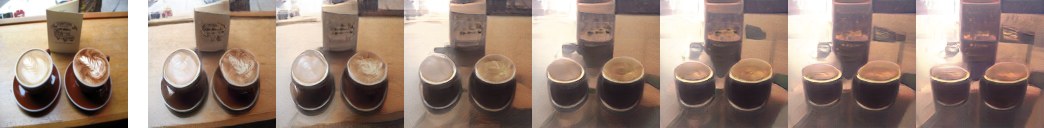}
  \includegraphics[width=0.49\linewidth]{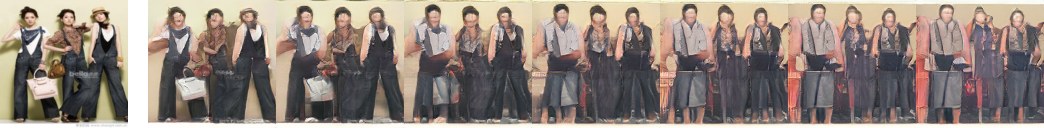}
  \hspace*{\fill}
  \includegraphics[width=0.49\linewidth]{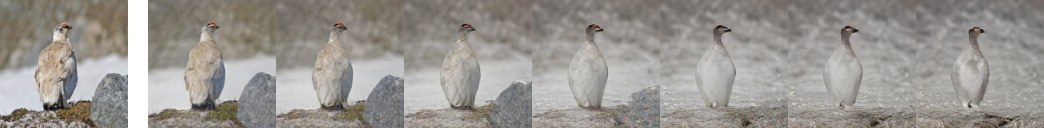}
  \includegraphics[width=0.49\linewidth]{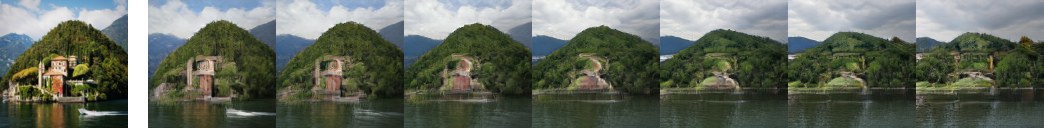}
  \hspace*{\fill}
  \includegraphics[width=0.49\linewidth]{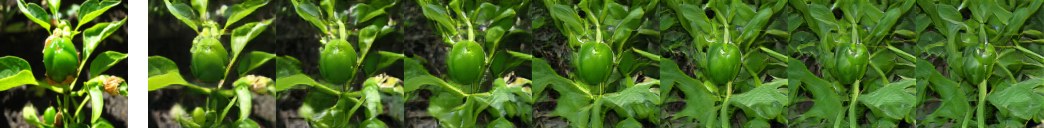}
  \includegraphics[width=0.49\linewidth]{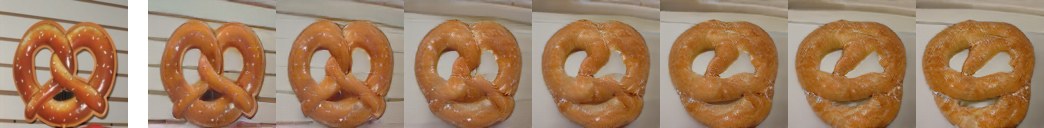}
  \hspace*{\fill}
  \includegraphics[width=0.49\linewidth]{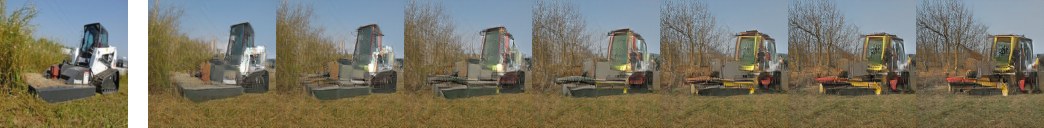}
  \includegraphics[width=0.49\linewidth]{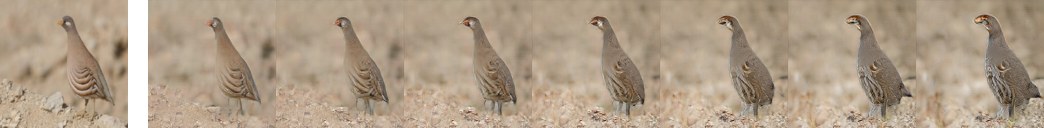}
  \hspace*{\fill}
  \includegraphics[width=0.49\linewidth]{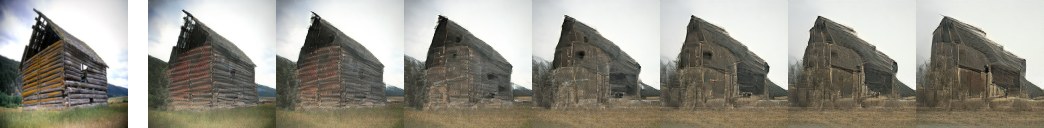}
  \includegraphics[width=0.49\linewidth]{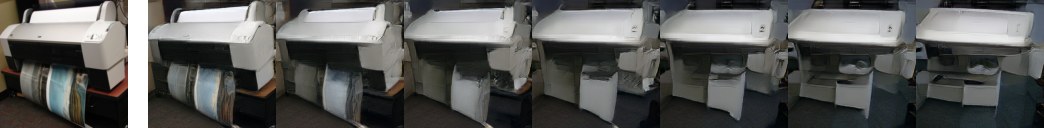}
  \hspace*{\fill}
  \includegraphics[width=0.49\linewidth]{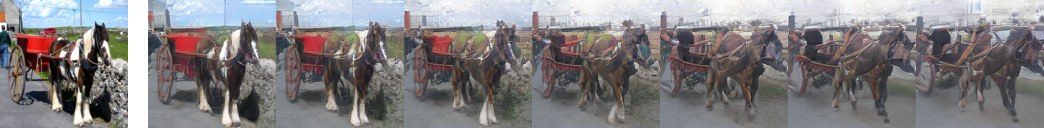}
  \includegraphics[width=0.49\linewidth]{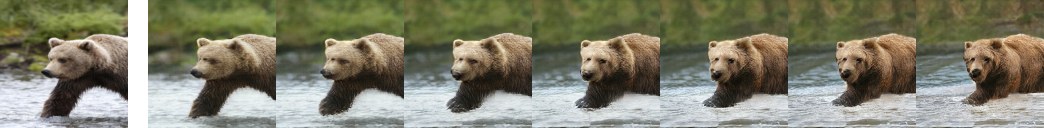}
  \hspace*{\fill}
  \includegraphics[width=0.49\linewidth]{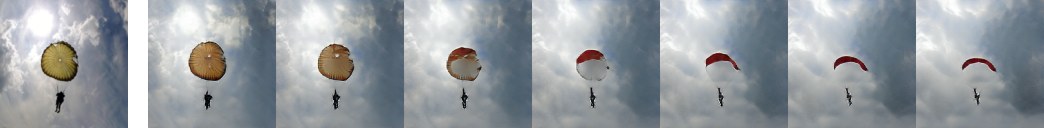}
  \includegraphics[width=0.49\linewidth]{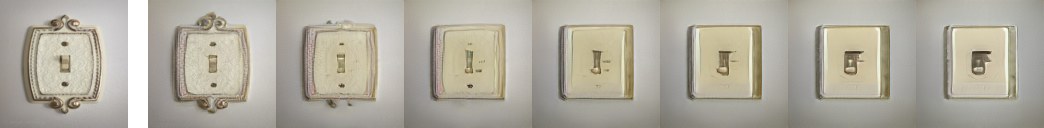}
  \hspace*{\fill}
  \includegraphics[width=0.49\linewidth]{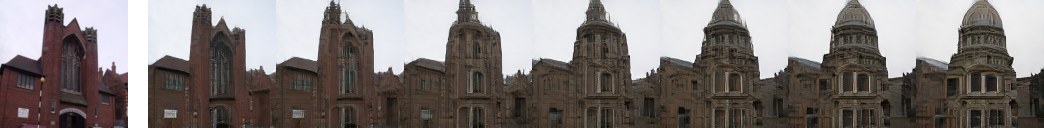}
  \includegraphics[width=0.49\linewidth]{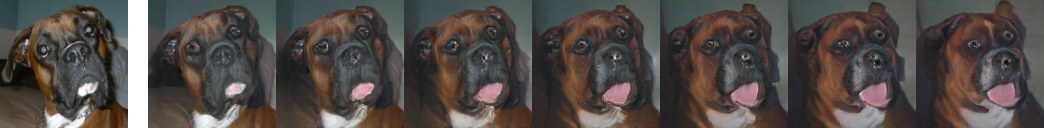}
  \hspace*{\fill}
  \includegraphics[width=0.49\linewidth]{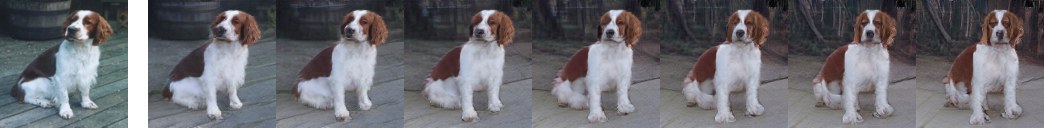}
  \includegraphics[width=0.49\linewidth]{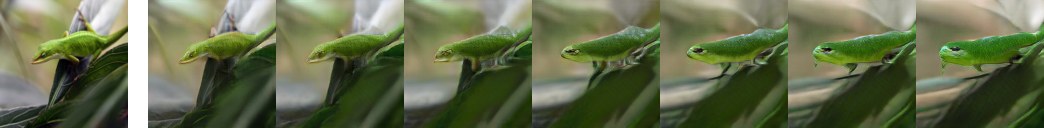}
  \hspace*{\fill}
  \includegraphics[width=0.49\linewidth]{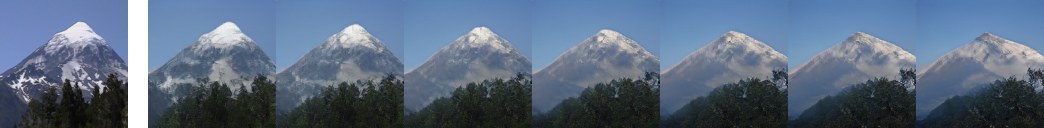}
  \includegraphics[width=0.49\linewidth]{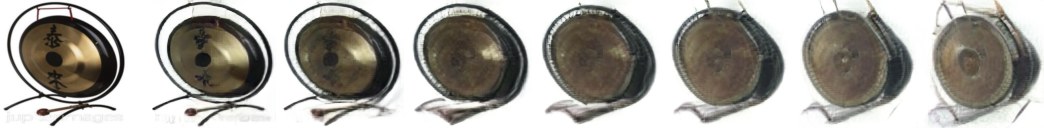}
  \hspace*{\fill}
  \includegraphics[width=0.49\linewidth]{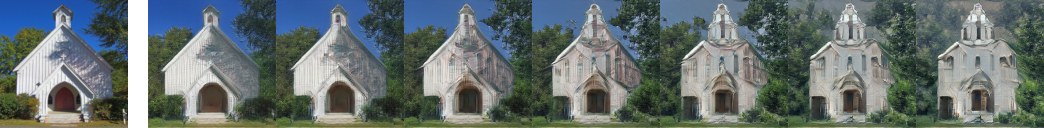}
  \includegraphics[width=0.49\linewidth]{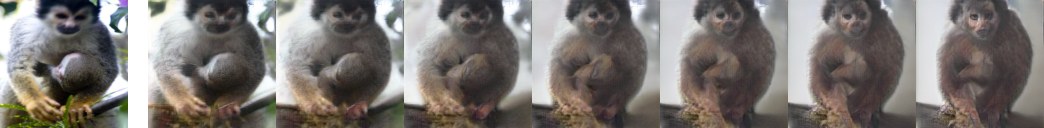}
  \hspace*{\fill}
  \includegraphics[width=0.49\linewidth]{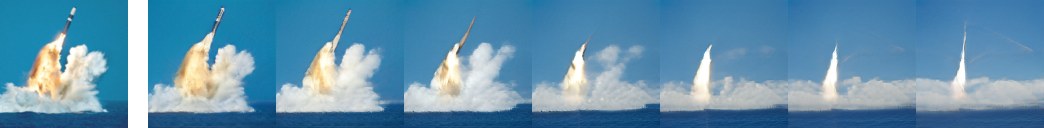}
  \includegraphics[width=0.49\linewidth]{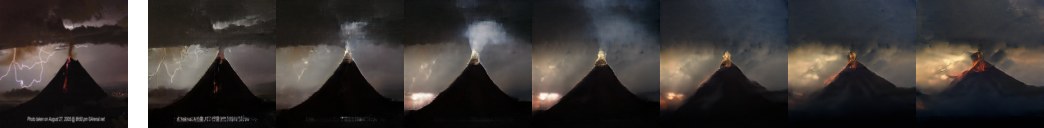}
  \hspace*{\fill}
  \includegraphics[width=0.49\linewidth]{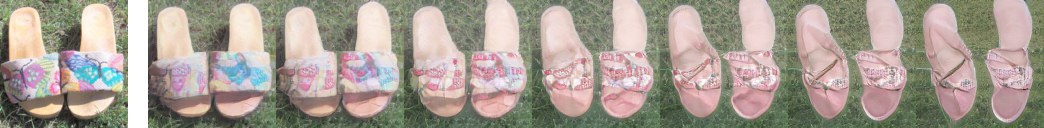}
  \includegraphics[width=0.49\linewidth]{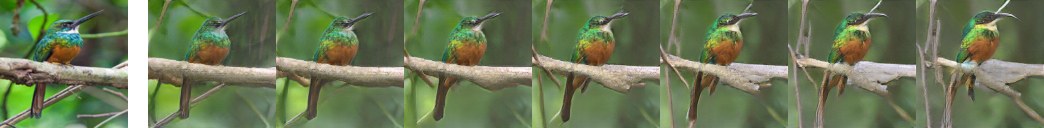}
  \hspace*{\fill}
  \includegraphics[width=0.49\linewidth]{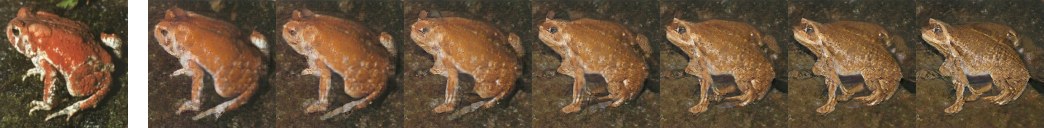}
  \includegraphics[width=0.49\linewidth]{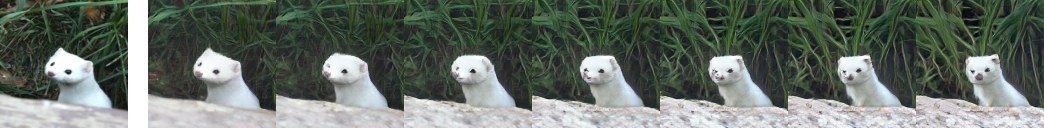}
  \hspace*{\fill}
  \includegraphics[width=0.49\linewidth]{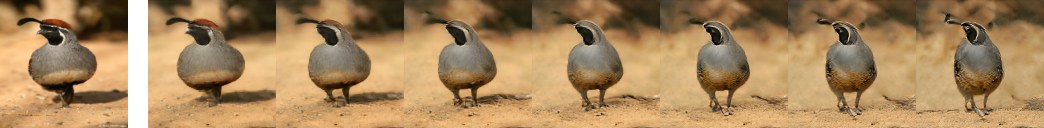}
  \includegraphics[width=0.49\linewidth]{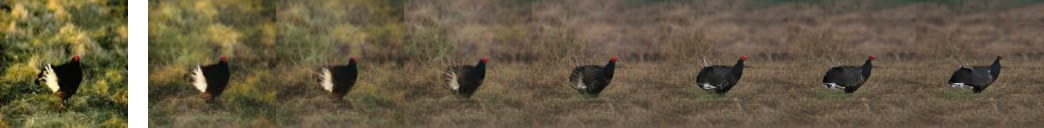}
  \hspace*{\fill}
  \includegraphics[width=0.49\linewidth]{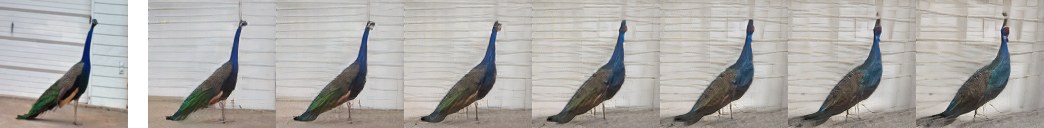}
  \includegraphics[width=0.49\linewidth]{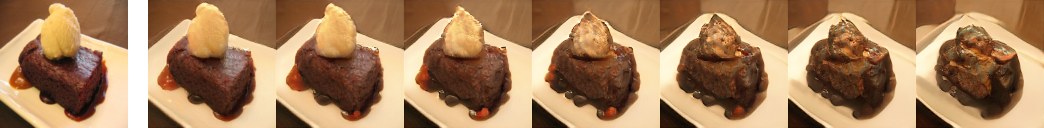}
  \hspace*{\fill}
  \includegraphics[width=0.49\linewidth]{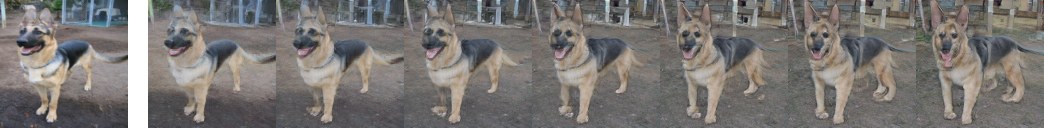}
  \includegraphics[width=0.49\linewidth]{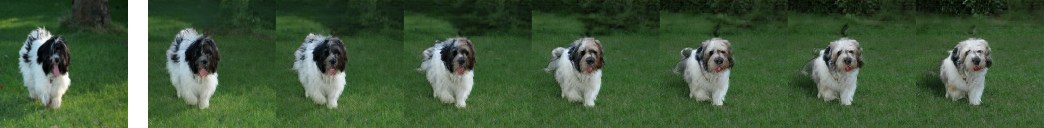}
  \hspace*{\fill}
  \includegraphics[width=0.49\linewidth]{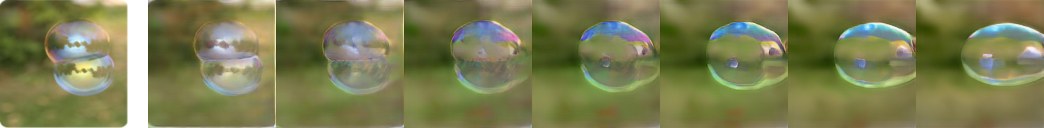}
  \caption{Left: Real image. Right: Linear interpolation between its closest
    reconstruction in the first step of the optimization
    ($G_2(G_1(z^*))$, right) and the closest reconstruction in the second step
    ($G_2(G_1(z^*)+\delta^*)$, left). Note that except from the right-most
    column, the rest of intermediate images can not be generated from
    the latent space.}
  \label{fig:delta_interpolation}
\end{figure}

%\section*{Further examples of Interpolations.}

\begin{figure}
  \centering
  \includegraphics[width=0.49\linewidth]{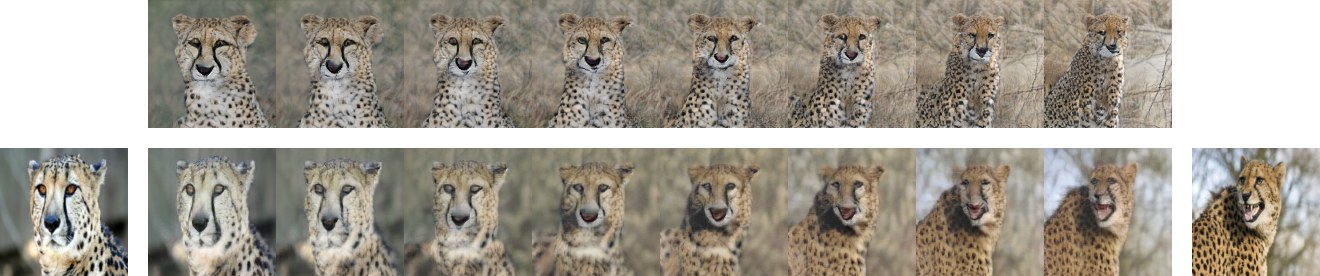}
  \hspace*{\fill}
  \includegraphics[width=0.49\linewidth]{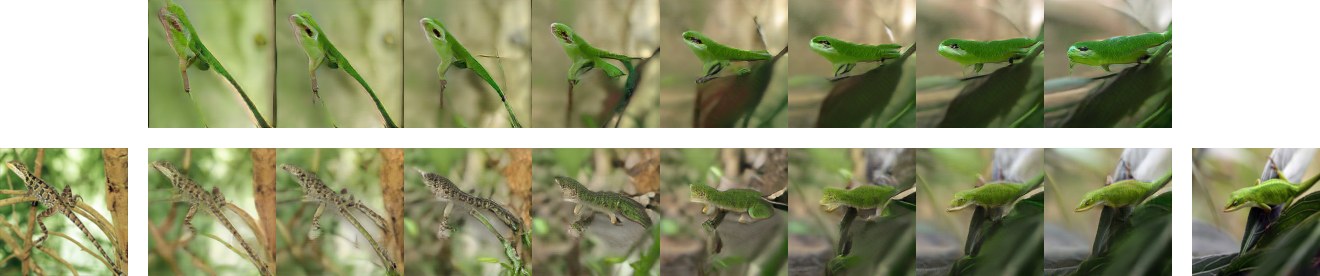}\\
  \vspace{1.5mm}
  \includegraphics[width=0.49\linewidth]{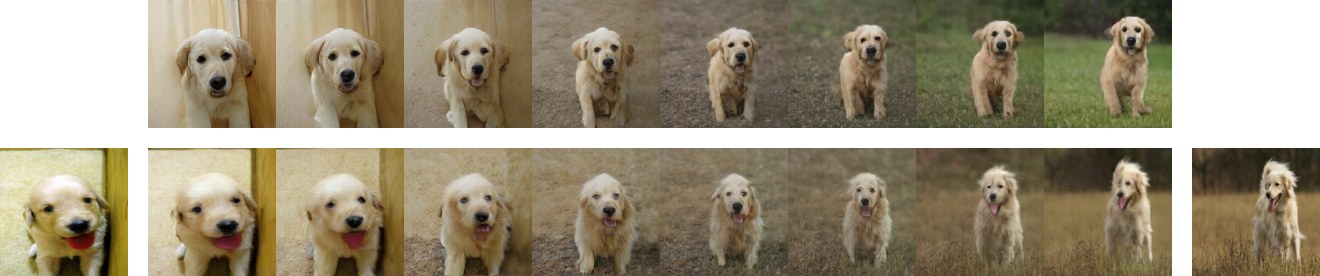}
  \hspace*{\fill}
  \includegraphics[width=0.49\linewidth]{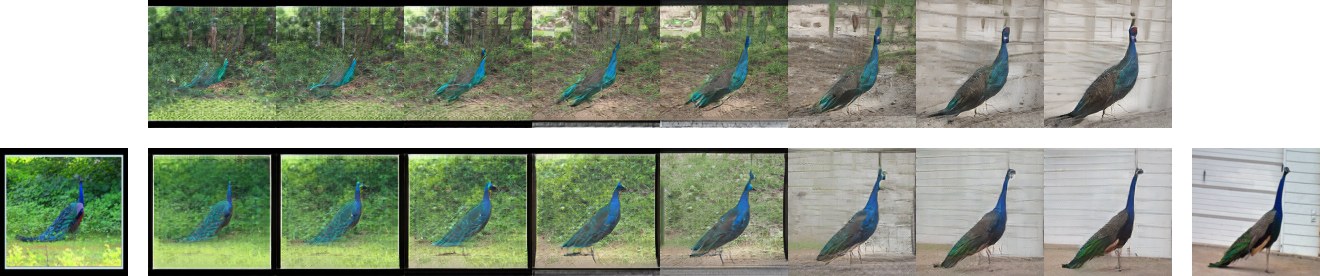}\\
  \vspace{1.5mm}
  \includegraphics[width=0.49\linewidth]{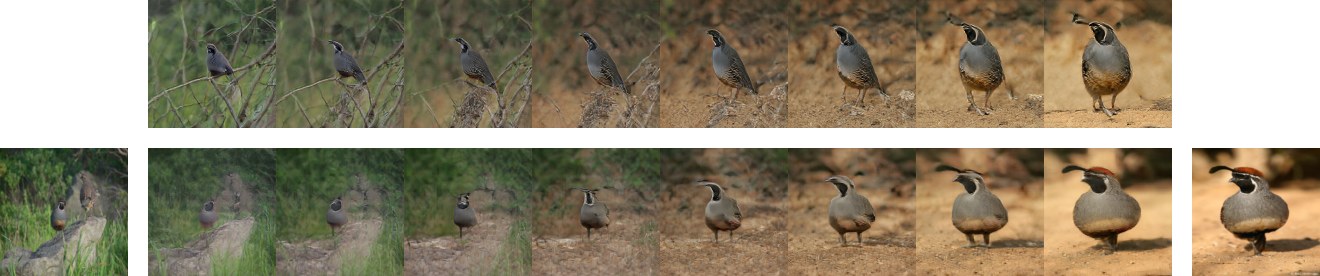}
  \hspace*{\fill}
  \includegraphics[width=0.49\linewidth]{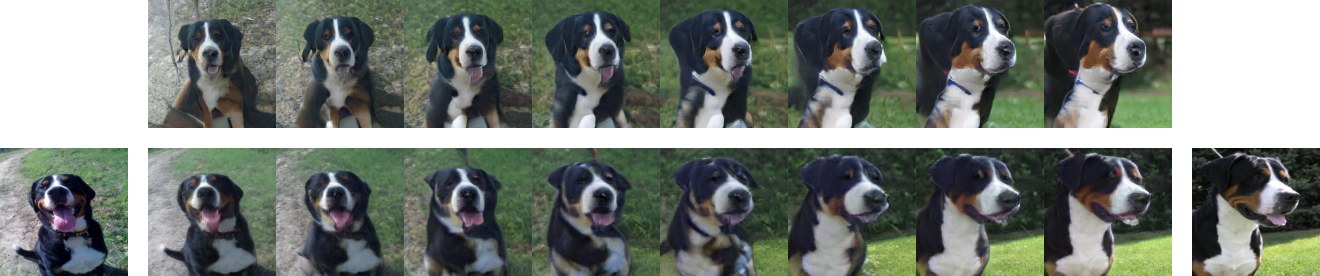}\\
  \vspace{1.5mm}
  \includegraphics[width=0.49\linewidth]{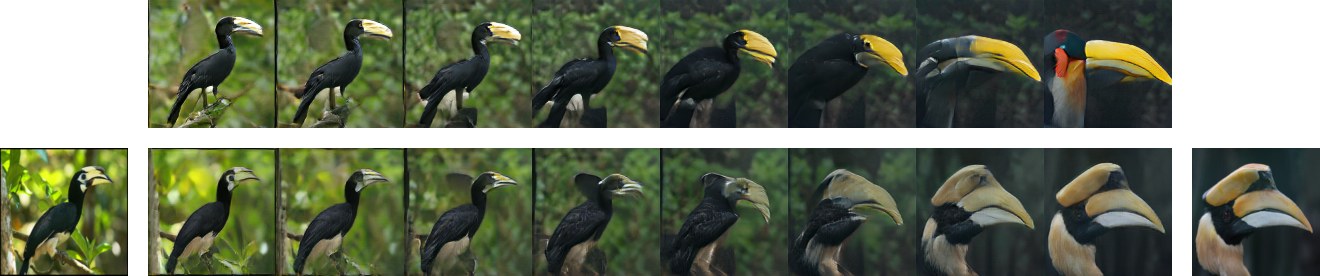}
  \hspace*{\fill}
  \includegraphics[width=0.49\linewidth]{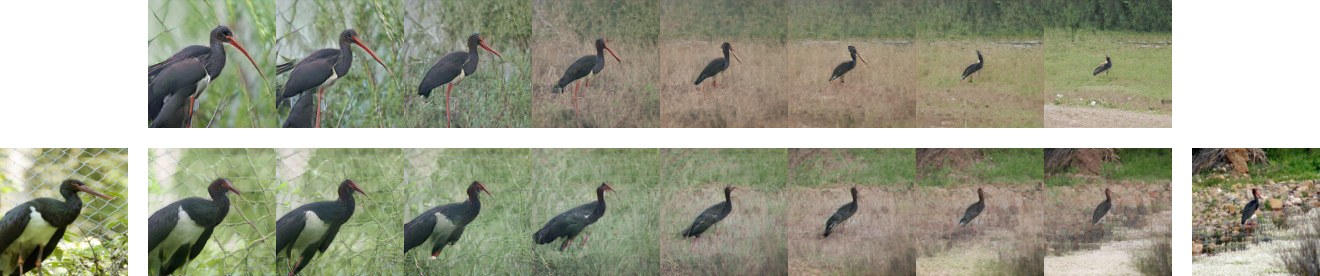}\\
  \vspace{1.5mm}
  \includegraphics[width=0.49\linewidth]{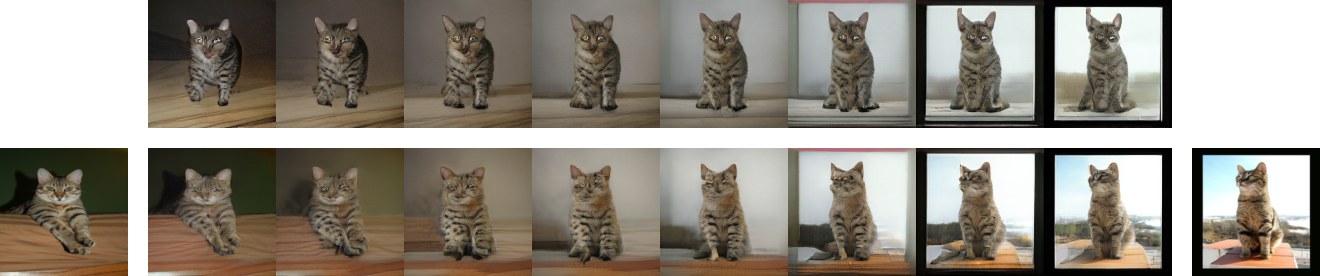}
  \hspace*{\fill}
  \includegraphics[width=0.49\linewidth]{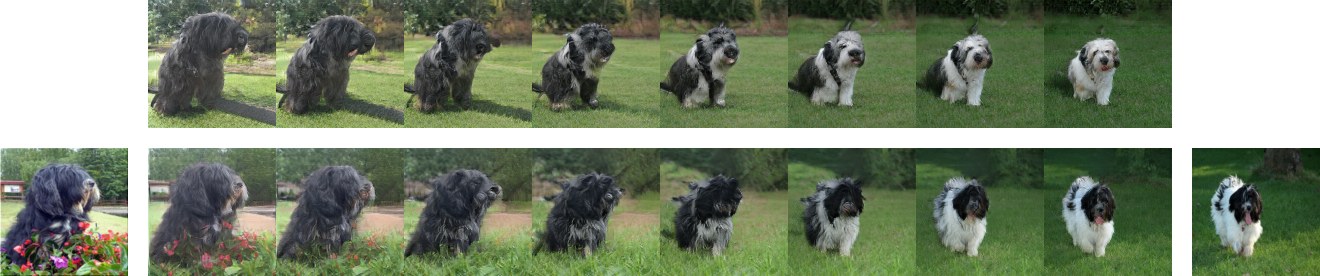}\\
  \vspace{1.5mm}
  \includegraphics[width=0.49\linewidth]{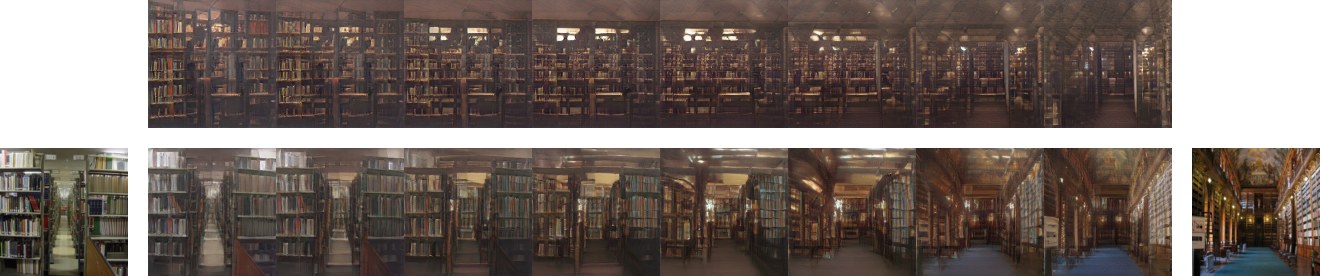}
  \hspace*{\fill}
  \includegraphics[width=0.49\linewidth]{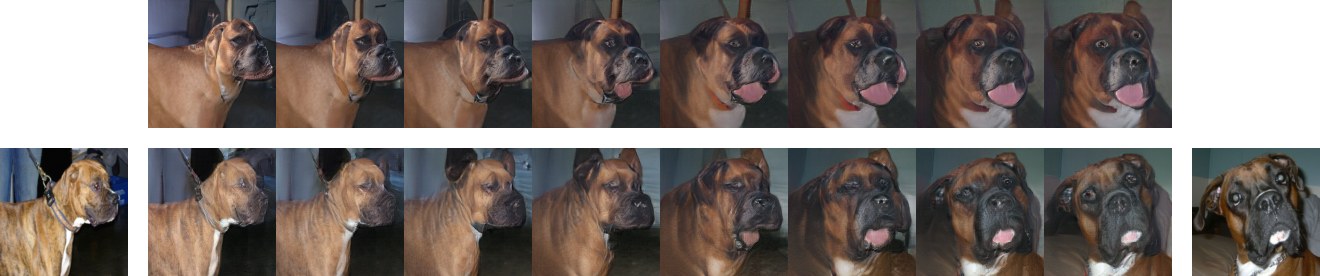}\\
  \vspace{1.5mm}
  \includegraphics[width=0.49\linewidth]{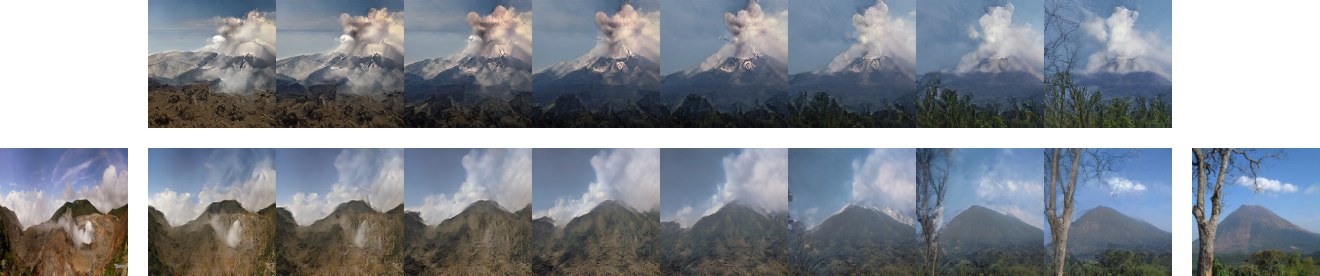}
  \hspace*{\fill}
  \includegraphics[width=0.49\linewidth]{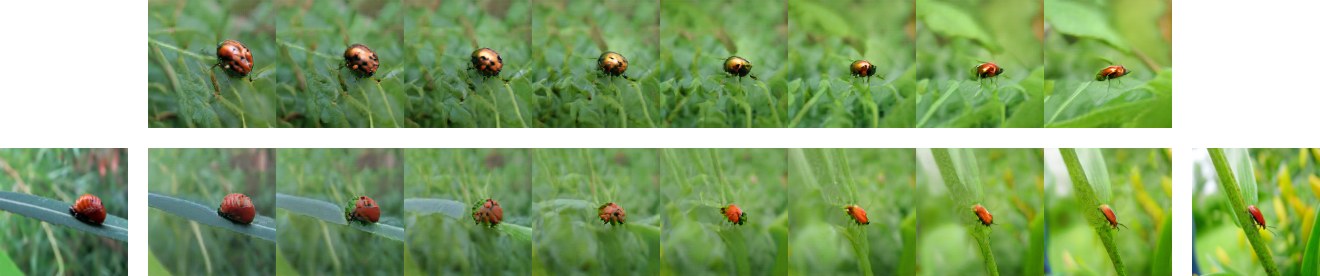}\
  \vspace{1.5mm}
  \includegraphics[width=0.49\linewidth]{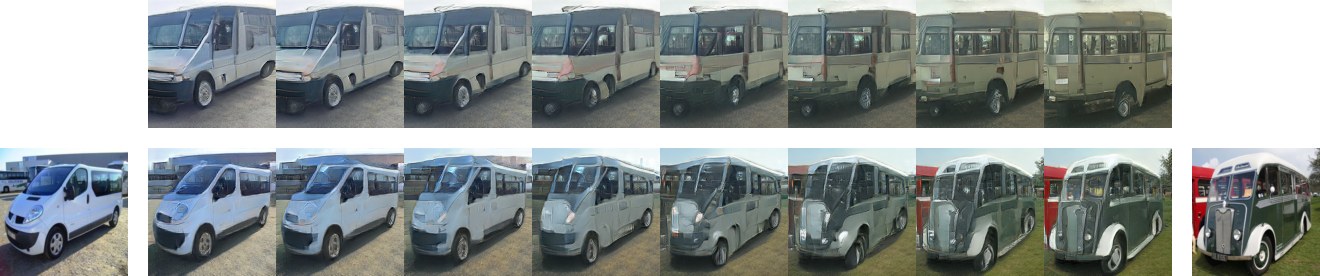}
  \hspace*{\fill}
  \includegraphics[width=0.49\linewidth]{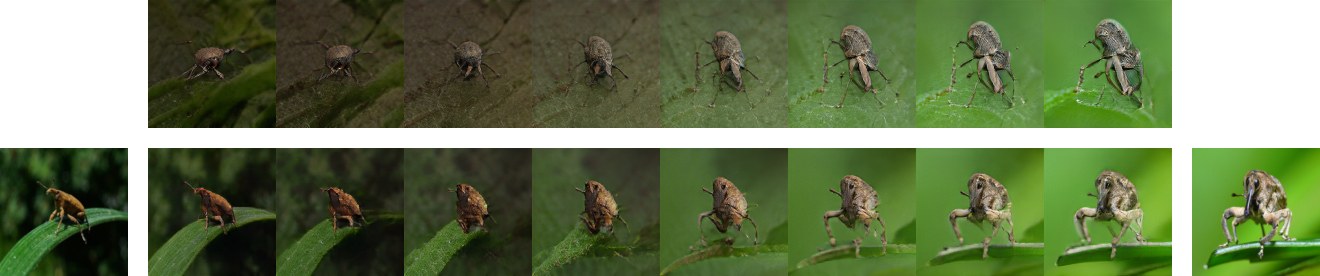}\\
  \vspace{1.5mm}
  \includegraphics[width=0.49\linewidth]{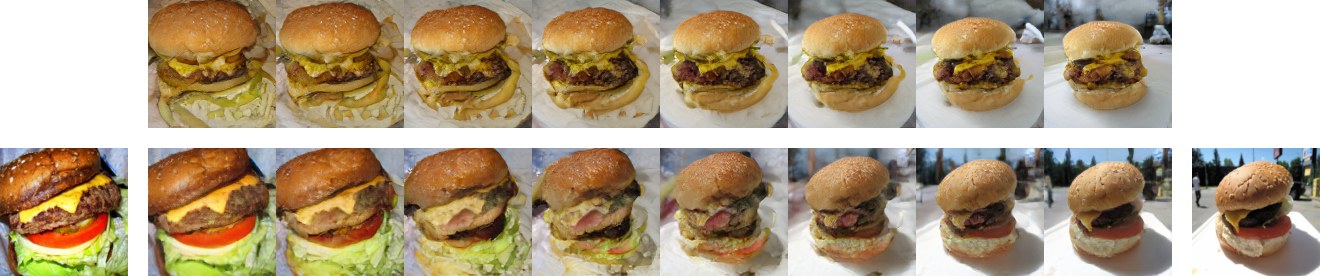}
  \hspace*{\fill}
  \includegraphics[width=0.49\linewidth]{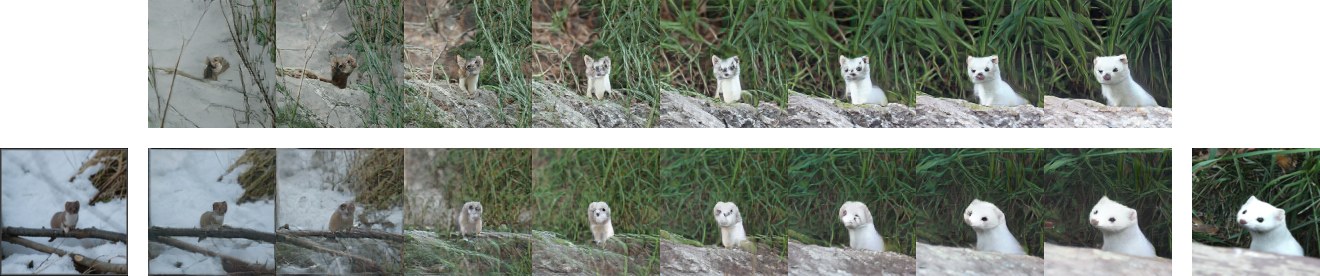}\\
  \vspace{1.5mm}
  \includegraphics[width=0.49\linewidth]{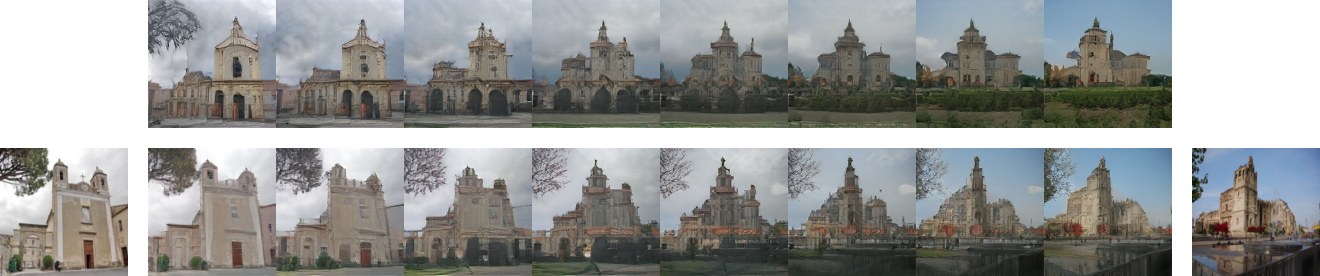}
  \hspace*{\fill}
  \includegraphics[width=0.49\linewidth]{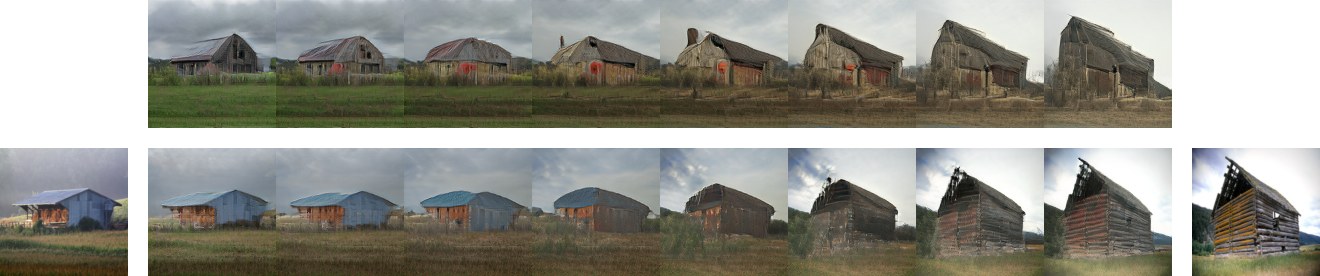}\\
  \vspace{1.5mm}
  \includegraphics[width=0.49\linewidth]{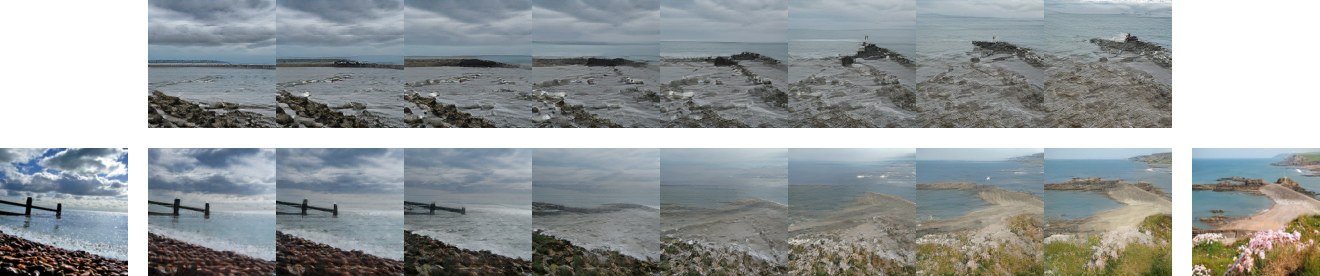}
  \hspace*{\fill}
  \includegraphics[width=0.49\linewidth]{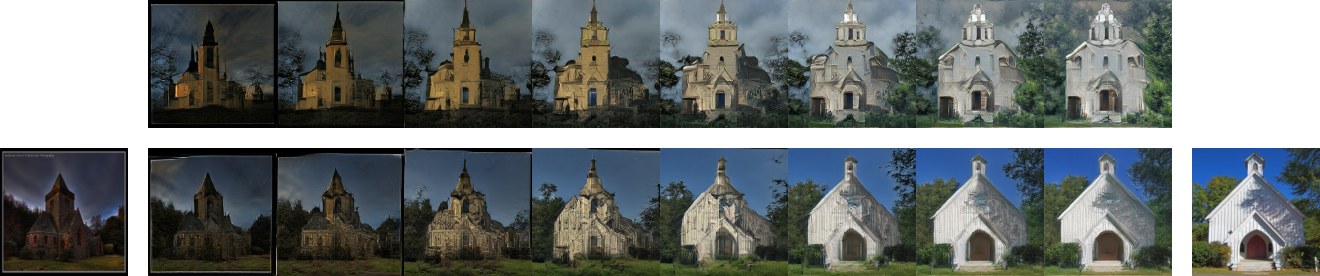}\\
  \vspace{1.5mm}
  \includegraphics[width=0.49\linewidth]{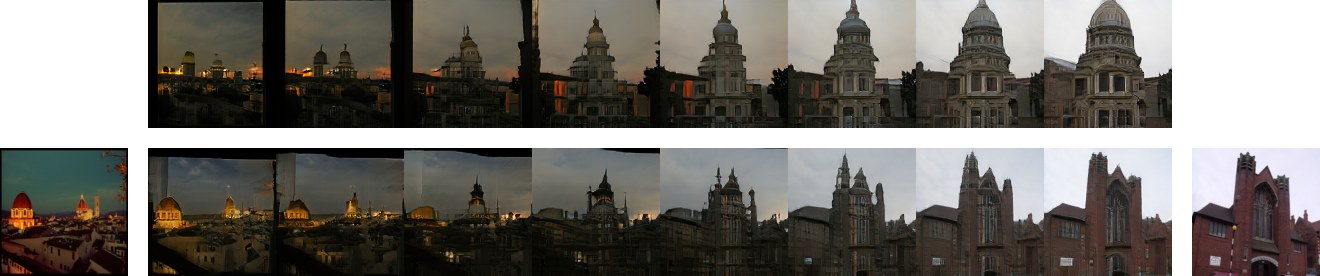}
  \hspace*{\fill}
  \includegraphics[width=0.49\linewidth]{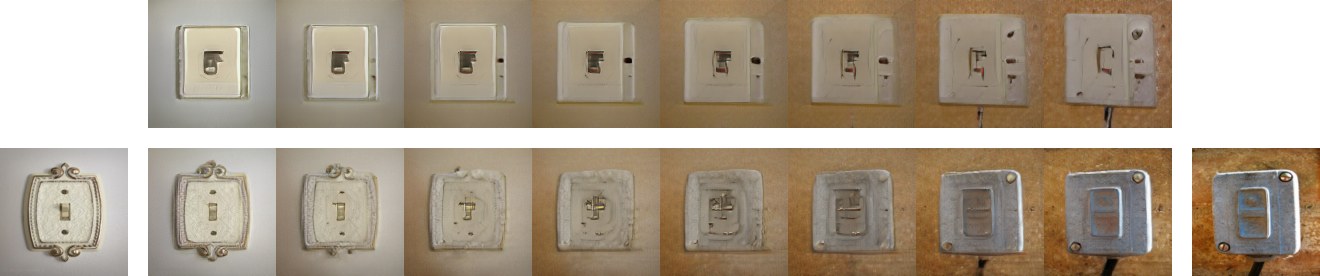}\\
  \vspace{1.5mm}
  \includegraphics[width=0.49\linewidth]{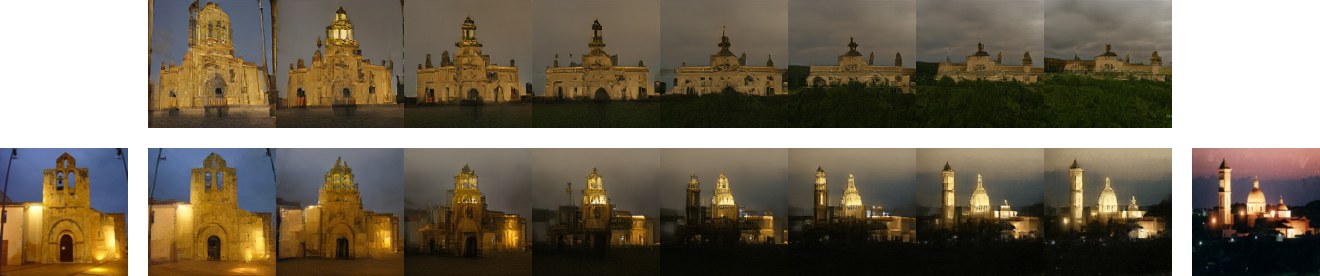}
  \hspace*{\fill}
  \includegraphics[width=0.49\linewidth]{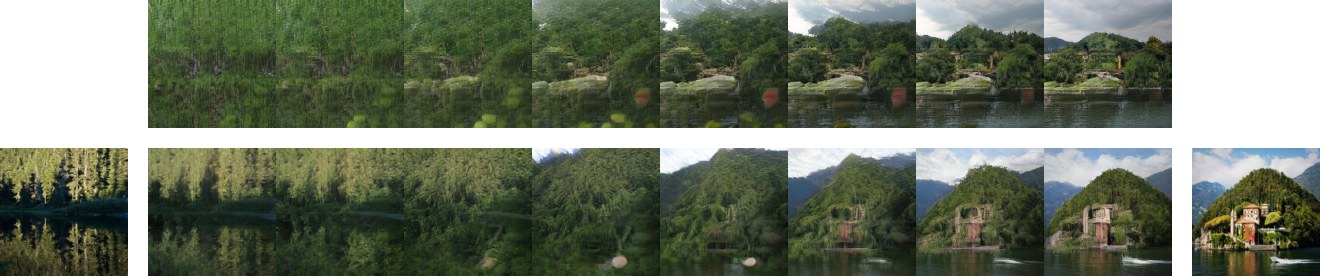}\\
  \caption{Linear interpolation between the reconstruction of two real images
    for the same class. First row: reconstruction in the latent space.
    Second row: reconstruction in the dense layer with the proposed
    two-step optimization.}
  \label{fig:linear_interp_sup}
\end{figure}

\begin{figure}
  \centering
  \includegraphics[width=0.49\linewidth]{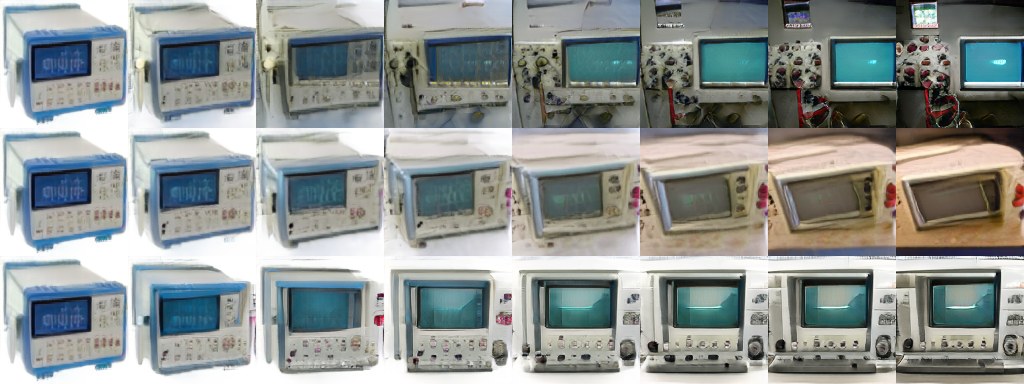}
  \includegraphics[width=0.49\linewidth]{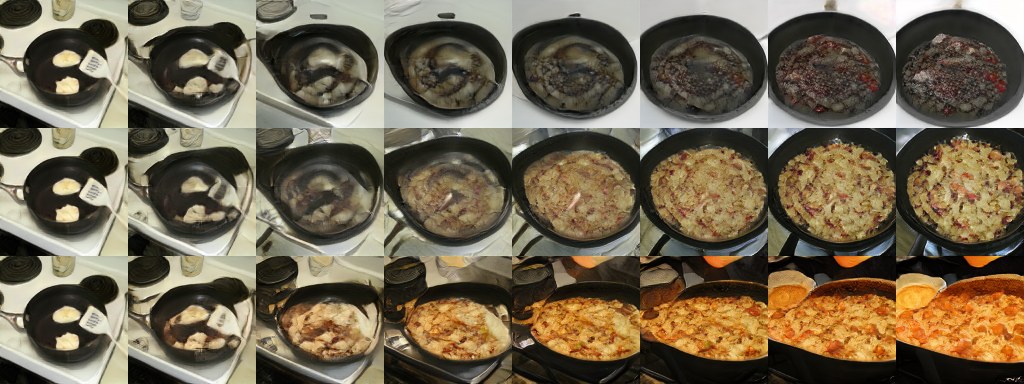}
  \includegraphics[width=0.49\linewidth]{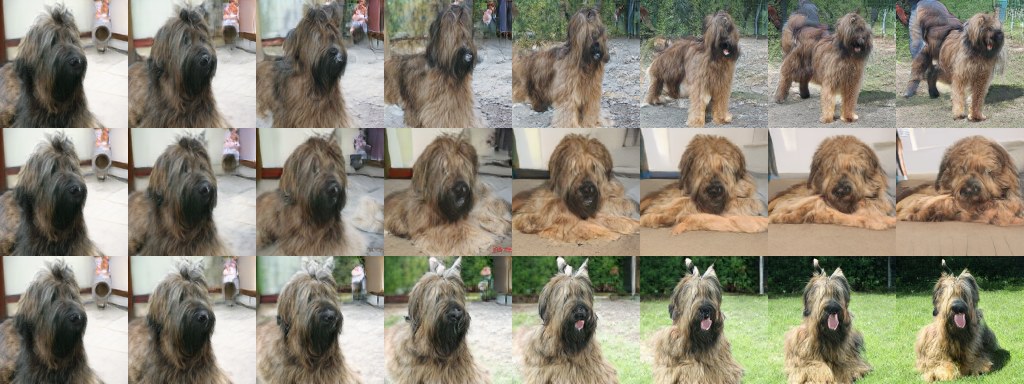}
  \includegraphics[width=0.49\linewidth]{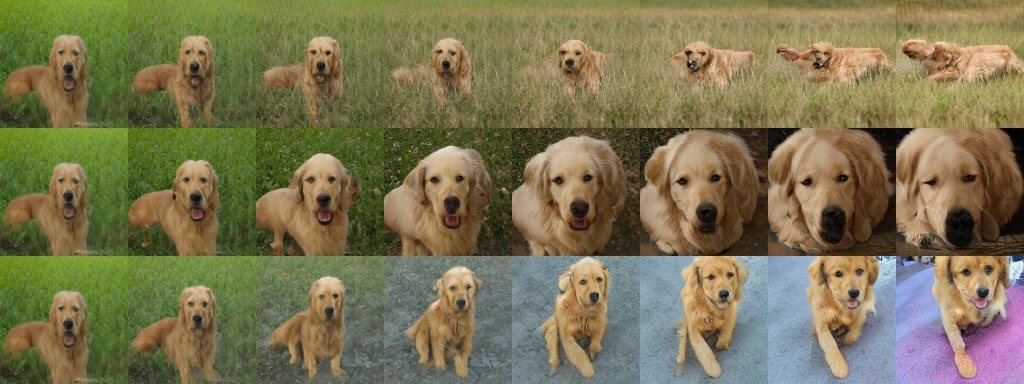}
  \includegraphics[width=0.49\linewidth]{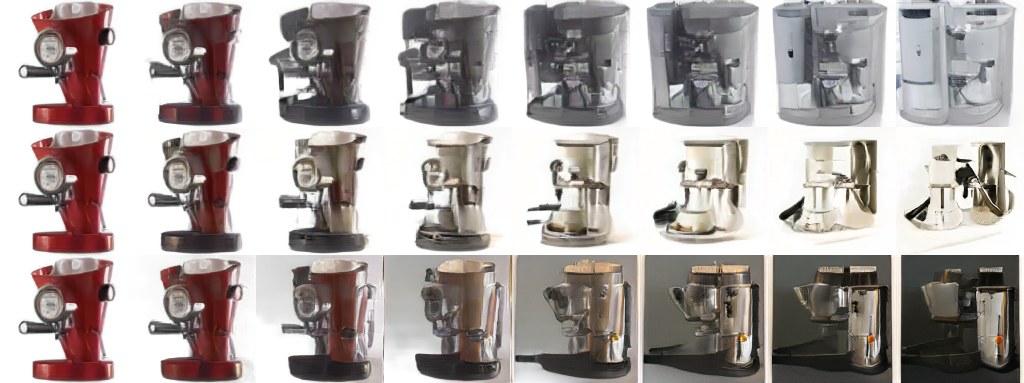}
  \includegraphics[width=0.49\linewidth]{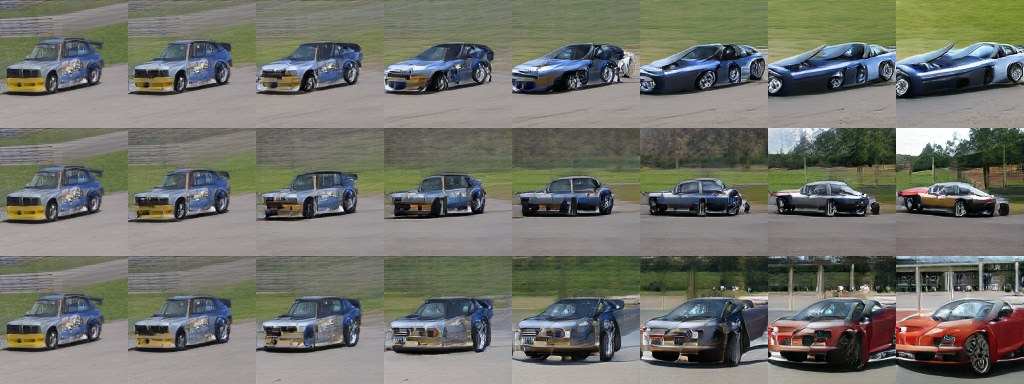}
  \includegraphics[width=0.49\linewidth]{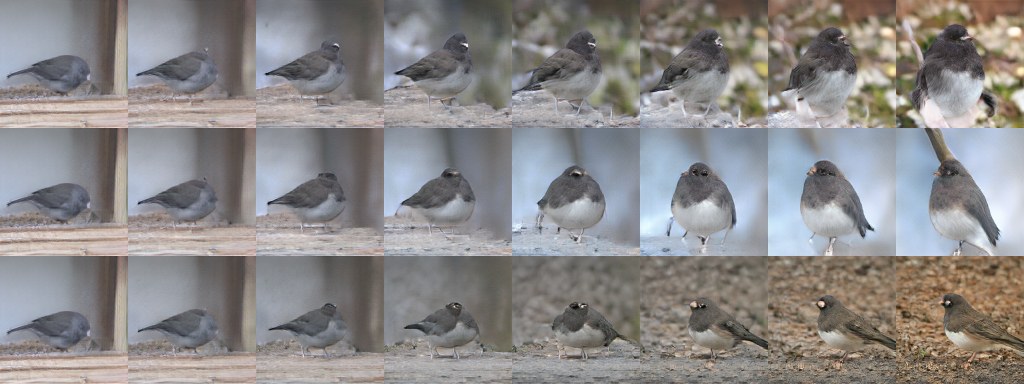}
  \includegraphics[width=0.49\linewidth]{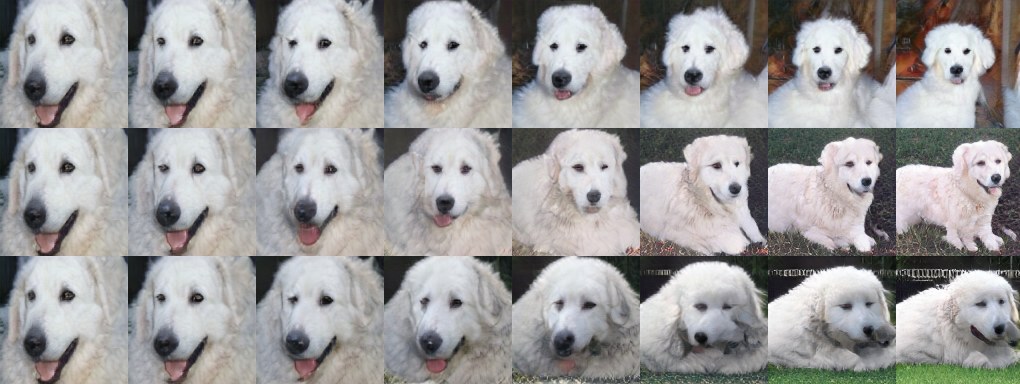}
  \includegraphics[width=0.49\linewidth]{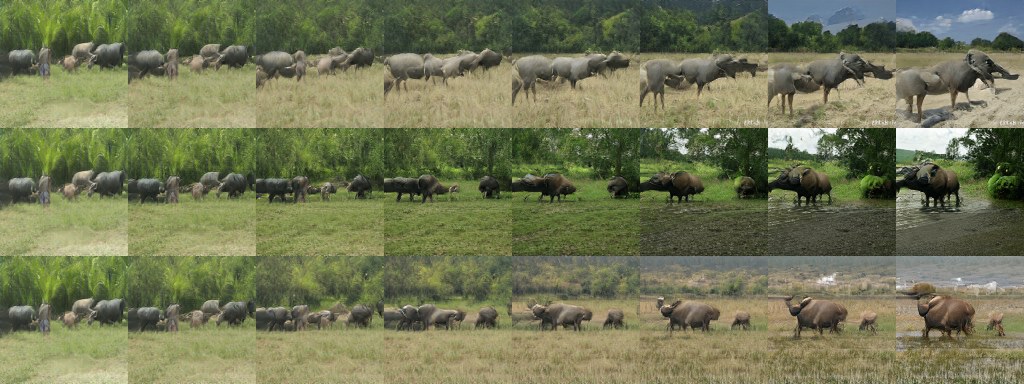}
  \includegraphics[width=0.49\linewidth]{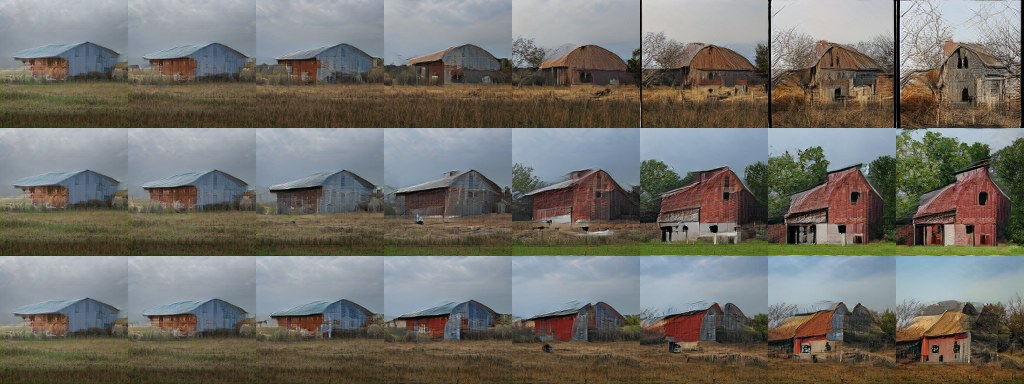}
  \includegraphics[width=0.49\linewidth]{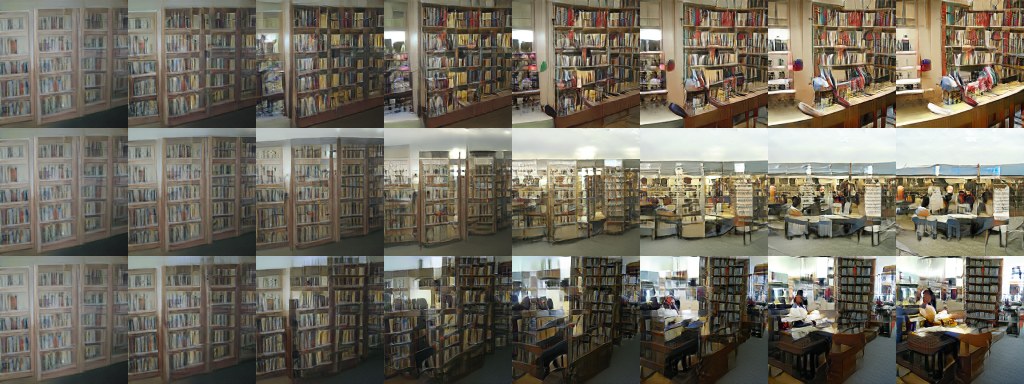}
  \includegraphics[width=0.49\linewidth]{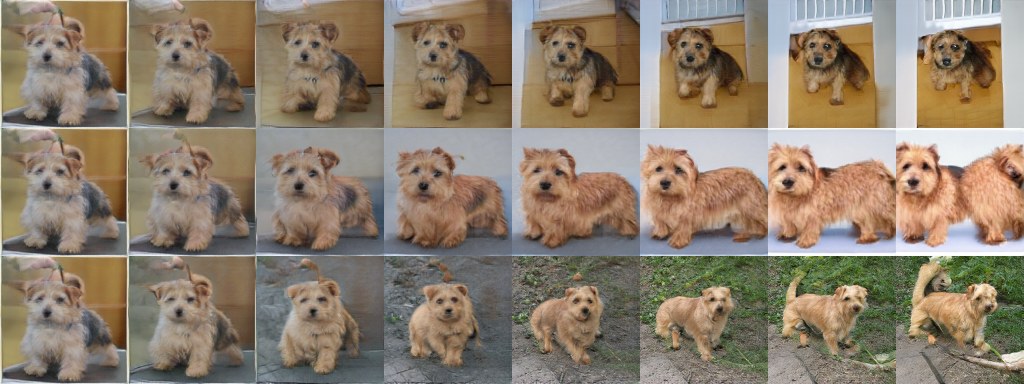}
  \caption{Linear interpolation in the space of the dense layer, between the
    reconstruction of a real image (left) and a random generated image in the
    same class (right). Note that except from the right-most column, the rest of
    intermediate images can not be generated from the latent space.}
  \label{fig:random_interpolation}
\end{figure}

%\section*{Further examples of Unsupervised Segmentation.}

\begin{figure}
  \centering
  \includegraphics[width=0.4\linewidth]{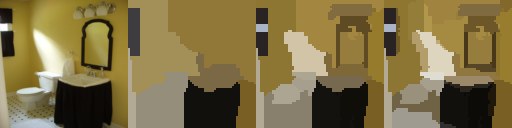}
  \hspace{0.007\linewidth}
  \includegraphics[width=0.4\linewidth]{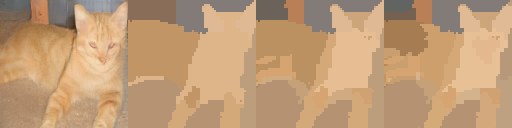}
  \includegraphics[width=0.4\linewidth]{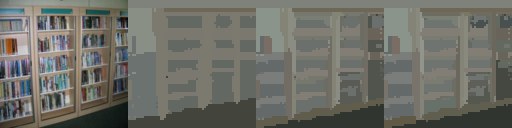}
  \hspace{0.007\linewidth}
  \includegraphics[width=0.4\linewidth]{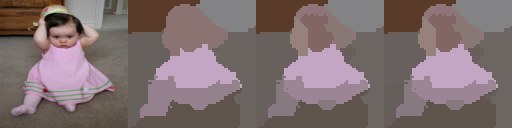}
  \includegraphics[width=0.4\linewidth]{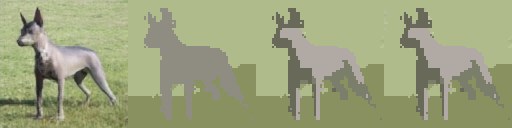}
  \hspace{0.007\linewidth}
  \includegraphics[width=0.4\linewidth]{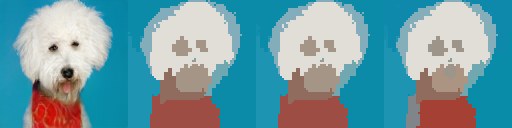}
  \includegraphics[width=0.4\linewidth]{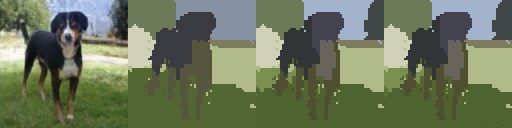}
  \hspace{0.007\linewidth}
  \includegraphics[width=0.4\linewidth]{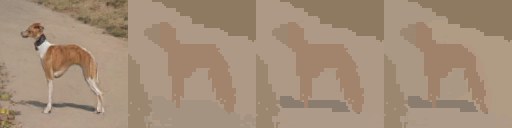}
  \includegraphics[width=0.4\linewidth]{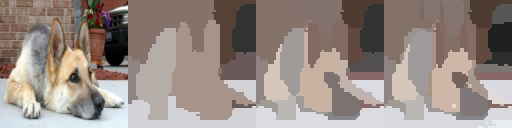}
  \hspace{0.007\linewidth}
  \includegraphics[width=0.4\linewidth]{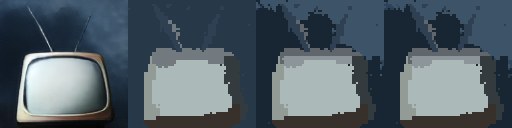}
  \includegraphics[width=0.4\linewidth]{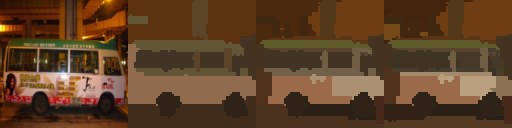}
  \hspace{0.007\linewidth}
  \includegraphics[width=0.4\linewidth]{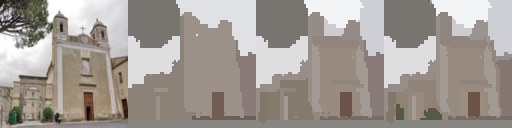}
  \includegraphics[width=0.4\linewidth]{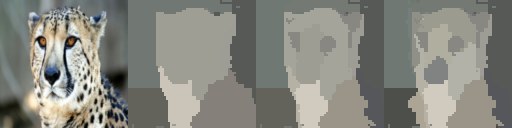}
  \hspace{0.007\linewidth}
  \includegraphics[width=0.4\linewidth]{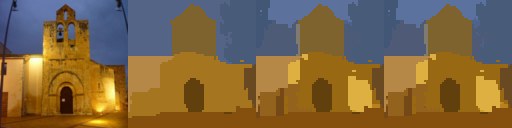}
  \includegraphics[width=0.4\linewidth]{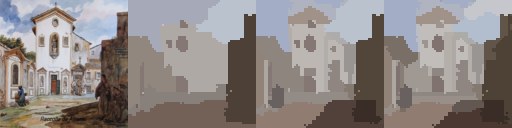}
  \hspace{0.007\linewidth}
  \includegraphics[width=0.4\linewidth]{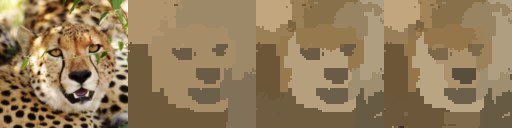}
  \includegraphics[width=0.4\linewidth]{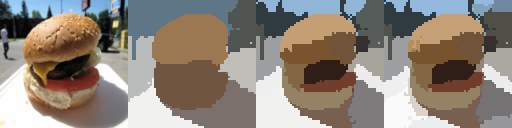}
  \hspace{0.007\linewidth}
  \includegraphics[width=0.4\linewidth]{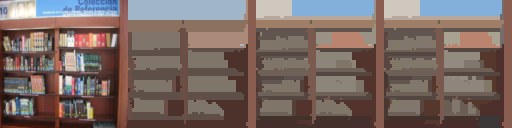}
  \includegraphics[width=0.4\linewidth]{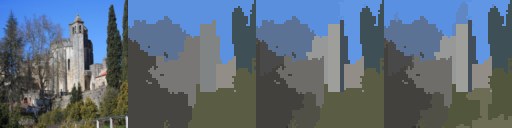}
  \hspace{0.007\linewidth}
  \includegraphics[width=0.4\linewidth]{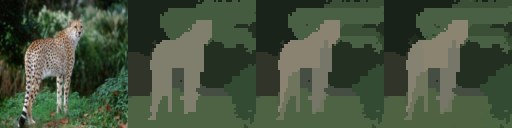}
  \includegraphics[width=0.4\linewidth]{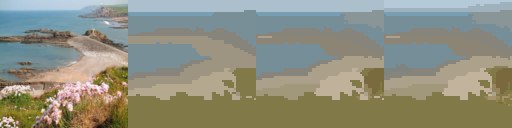}
  \hspace{0.007\linewidth}
  \includegraphics[width=0.4\linewidth]{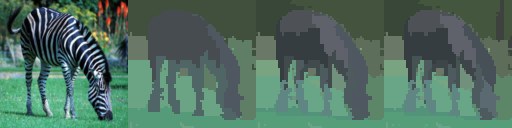}
  \includegraphics[width=0.4\linewidth]{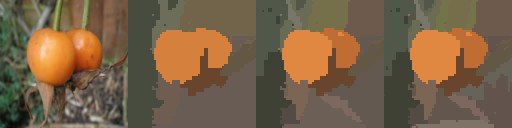}
  \hspace{0.007\linewidth}
  \includegraphics[width=0.4\linewidth]{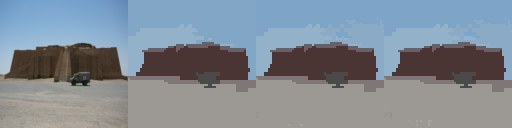}
  \includegraphics[width=0.4\linewidth]{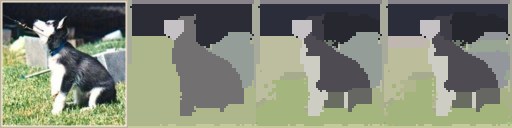}
  \hspace{0.007\linewidth}
  \includegraphics[width=0.4\linewidth]{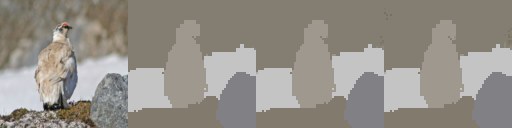}
  \includegraphics[width=0.4\linewidth]{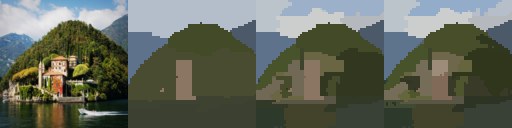}
  \hspace{0.007\linewidth}
  \includegraphics[width=0.4\linewidth]{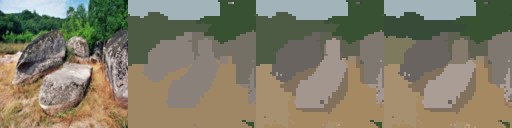}
  \includegraphics[width=0.4\linewidth]{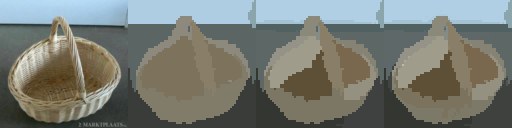}
  \hspace{0.007\linewidth}
  \includegraphics[width=0.4\linewidth]{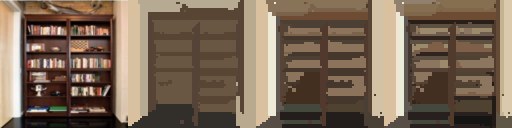}
  \caption{Examples of unsupervised segmentation (64x64) for real images, after
    inverting the generator to the first layer
    (different numbers of clusters: 8, 20, 40). Only for visualization purposes,
    each cluster is associated to the average color of all its pixel members.}.
  \label{fig:segmentation_sup}
\end{figure}

\section*{Additional experiments}
We replicated the experiments with different models and datasets: the Progressive GAN generator trained on CelebA-HQ and the Improved WGAN unconditional generator with DCGAN architecture trained on CIFAR-10. In both cases we obtain equivalent results (see for example Figure \ref{fig:interpolation_images}).

\section*{Additional applications}
Besides the proposed unsupervised segmentation, our method can be applied for manipulation of general \textit{real} images (e.g. interpolation between images, class switching, etc.). See Figs. \ref{fig:interpolation_images}, \ref{fig:frames_video} and third row of Fig. \ref{fig:reg}.

\begin{figure}[t]
  \centering
  \begin{minipage}{.485\textwidth}
    \centering
    \includegraphics[width=\linewidth]{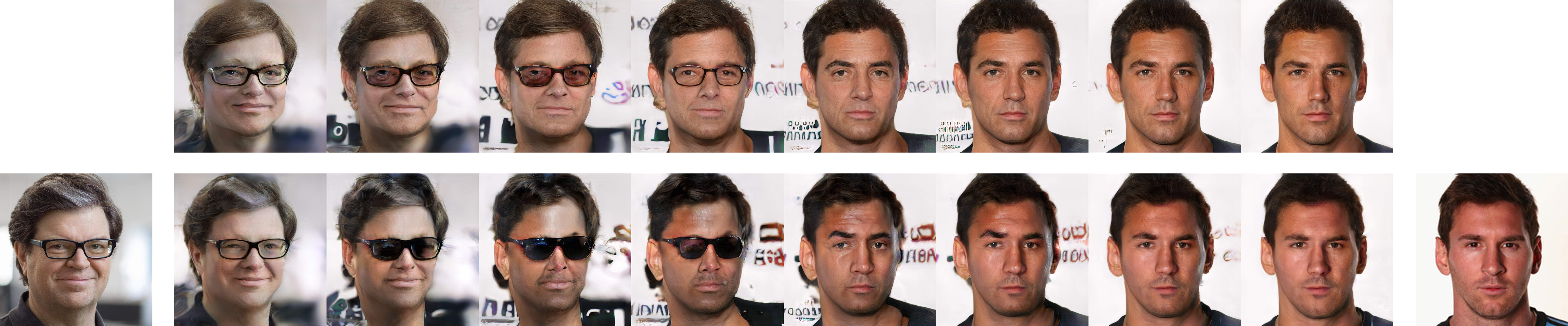}\\
    \vspace{2mm}
    \includegraphics[width=\linewidth]{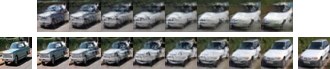}
    \captionof{figure}{Linear interpolation between the reconstruction of two real images (First row: the latent space. Second row: the dense layer). Top: Progressive GAN on CelebA-HQ. Bottom: Improved WGAN on CIFAR-10.}
    \label{fig:interpolation_images}
  \end{minipage}
  \hspace*{\fill}
  \begin{minipage}{.48\textwidth}
    \centering
    \includegraphics[width=0.88\linewidth]{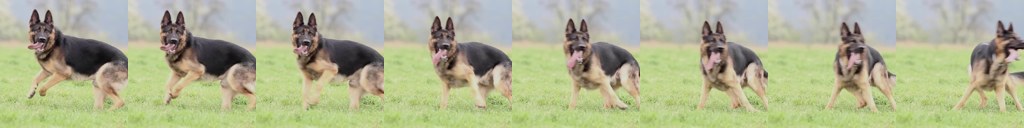}\\
    \includegraphics[width=0.88\linewidth]{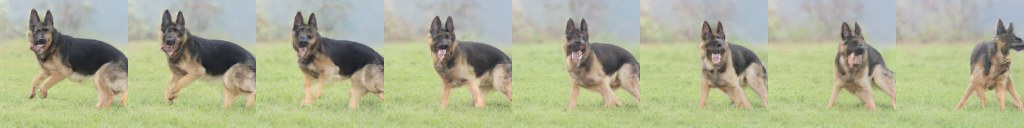}
    \includegraphics[width=0.88\linewidth]{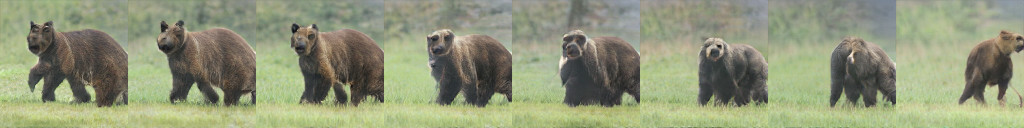}\\
    \includegraphics[width=0.88\linewidth]{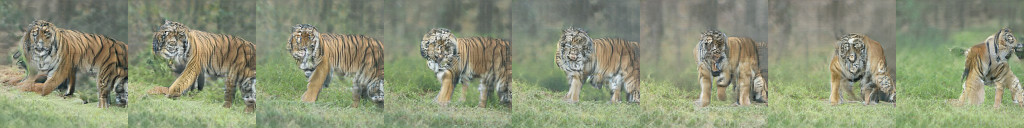}
    \captionof{figure}{Illustrative practical application of our method (BigGAN). First row: original sequence of frames. Second row: reconstruction in the first dense layer. Third and Fourth row: reconstruction after changing the input class.}
    \label{fig:frames_video}
  \end{minipage}
  
  \vspace{4mm}
  \begin{minipage}{.485\textwidth}
  \centering
  \includegraphics[width=0.11\linewidth]{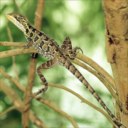}
  \includegraphics[width=0.11\linewidth]{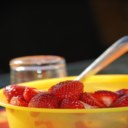}
  \includegraphics[width=0.11\linewidth]{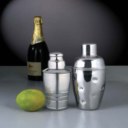}
  \includegraphics[width=0.11\linewidth]{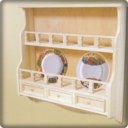}
  \includegraphics[width=0.11\linewidth]{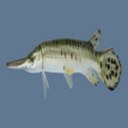}
  \includegraphics[width=0.11\linewidth]{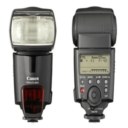}
  \includegraphics[width=0.11\linewidth]{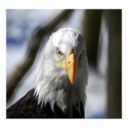}
  \includegraphics[width=0.11\linewidth]{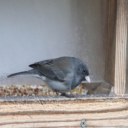}\\
  \vspace{1mm}
  \includegraphics[width=0.11\linewidth]{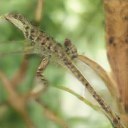}
  \includegraphics[width=0.11\linewidth]{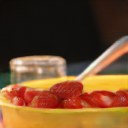}
  \includegraphics[width=0.11\linewidth]{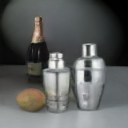}
  \includegraphics[width=0.11\linewidth]{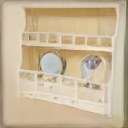}
  \includegraphics[width=0.11\linewidth]{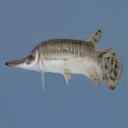}
  \includegraphics[width=0.11\linewidth]{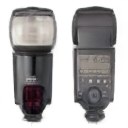}
  \includegraphics[width=0.11\linewidth]{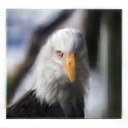}
  \includegraphics[width=0.11\linewidth]{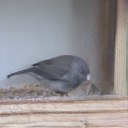}\\
  \vspace{1mm}
  \includegraphics[width=0.11\linewidth]{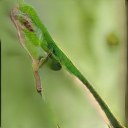}
  \includegraphics[width=0.11\linewidth]{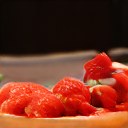}
  \includegraphics[width=0.11\linewidth]{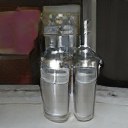}
  \includegraphics[width=0.11\linewidth]{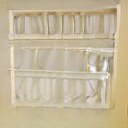}
  \includegraphics[width=0.11\linewidth]{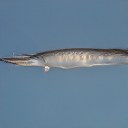}
  \includegraphics[width=0.11\linewidth]{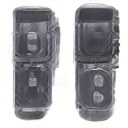}
  \includegraphics[width=0.11\linewidth]{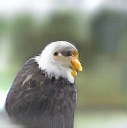}
  \includegraphics[width=0.11\linewidth]{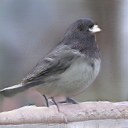}\\
  \captionof{figure}{Top: Real images. Middle: Reconstruction in the dense layer ($h^*$). Bottom: proj of $h^*$ onto $G_1(\Z)$.}
  \label{fig:projection}
  \end{minipage}
  \hspace*{\fill}
  \begin{minipage}{.485\textwidth}
    \centering
    \includegraphics[width=0.88\linewidth]{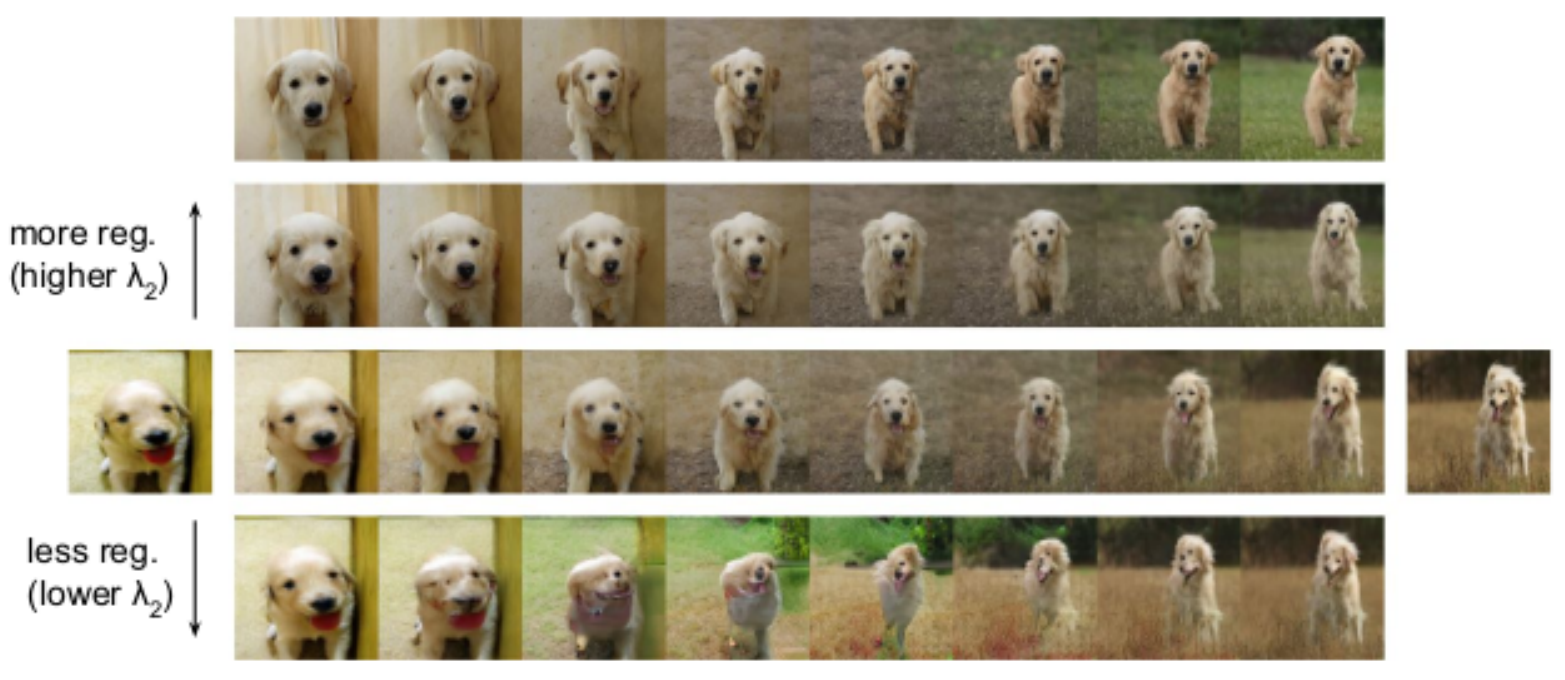}
    \captionof{figure}{Effects of regularization in the interpolation quality (top row: interpolation in the latent space).}
    \label{fig:reg}
  \end{minipage}
\end{figure}

\end{document}